\documentclass[10pt]{article}
\usepackage{times}
\pdfoutput=1

\usepackage{authblk}
\usepackage{algorithm}
\usepackage{algorithmic}
\usepackage{parskip}

\newlength{\defbaselineskip}
\usepackage{etoolbox}
\newtoggle{arxiv}
\toggletrue{arxiv}

\usepackage[numbers,sort]{natbib}

\usepackage{amsmath, amsthm, amssymb}
\usepackage[titletoc,toc,page]{appendix}
\usepackage{epstopdf}
\usepackage{hyperref}
\usepackage[capitalise]{cleveref}       %
\hypersetup{colorlinks,linkcolor=blue,citecolor=blue,urlcolor=blue}
\usepackage{mathtools}
\usepackage{multirow}
\usepackage{comment}
\usepackage[inline]{enumitem}
\usepackage{caption}
\usepackage{subcaption}
\usepackage{microtype}
\usepackage{graphicx}
\usepackage{booktabs} %

\usepackage{xcolor}
\usepackage{thmtools}
\usepackage{thm-restate}

\newtheorem*{remark*}{Remark}
\newtheorem*{cor*}{Corollary}

\declaretheorem[name=Theorem,numberwithin=section]{theorem}
\declaretheorem[name=Proposition,sharenumber=theorem]{prop}
\declaretheorem[name=Lemma,sharenumber=theorem]{lemma}
\declaretheorem[name=Corollary,sharenumber=theorem]{cor}

\declaretheoremstyle[headfont=\bfseries,bodyfont=\normalfont,style=definition]{}
\declaretheorem[name=Definition,sharenumber=theorem,style=definition]{definition}

\declaretheoremstyle[bodyfont=\normalfont,style=remark]{}
\declaretheorem[name=Remark,sharenumber=theorem,style=remark]{remark}

\newcommand{\name}{\textsc{HoroPCA}}

\newcommand{\piHoro}{\pi^{\operatorname{H}}}
\newcommand{\piGeo}{\pi^{\operatorname{G}}}
\newcommand{\piEuc}{\pi^{\operatorname{E}}}
\newcommand{\piMinkowski}{\pi^{\operatorname{Mink}}}

\newcommand{\R}{\mathbb{R}}

\renewcommand{\S}{\mathbb{S}}
\newcommand{\Sinf}{\S_\infty}
\renewcommand{\H}{\mathbb{H}}

\DeclareMathOperator{\arccosh}{cosh^{-1}}

\DeclareMathOperator{\arctanh}{tanh^{-1}}
\DeclareMathOperator*{\argmin}{argmin}
\DeclareMathOperator*{\argmax}{argmax}

\DeclareMathOperator{\GH}{GH}


\begin{document}

\title{HoroPCA: Hyperbolic Dimensionality Reduction via\\ Horospherical Projections}

\author[$\ddagger$]{Ines Chami\thanks{Equal contribution.}}
\author[$\dagger$]{Albert Gu$^*$}
\author[$\mathsection$]{Dat Nguyen$^*$}
\author[$\dagger$]{Christopher R{\'e}}
\affil[$\dagger$]{Department of Computer Science, Stanford University}
\affil[$\ddagger$]{Institute of Computational and Mathematical Engineering, Stanford University}
\affil[$\mathsection$]{Department of Mathematics, Stanford University\vspace{4pt}}

\affil[ ]{\footnotesize{\texttt{chami@cs.stanford.edu}, \texttt{albertgu@cs.stanford.edu}, \texttt{datpn2@gmail.com},\texttt{chrismre@cs.stanford.edu}}}

\maketitle

\begin{abstract}
This paper studies Principal Component Analysis (PCA) for data lying in hyperbolic spaces.
Given directions, PCA relies on: (1) a parameterization of subspaces spanned by these directions, (2) a method of \emph{projection} onto subspaces that preserves information in these directions, and (3) an \emph{objective} to optimize, namely the variance explained by projections.
We generalize each of these concepts to the hyperbolic space and propose \name{}, a method for hyperbolic dimensionality reduction.
By focusing on the core problem of extracting \emph{principal directions}, \name{} theoretically better preserves information in the original data such as distances, compared to previous generalizations of PCA.
Empirically, we validate that \name{} outperforms existing dimensionality reduction methods, significantly reducing error in distance preservation. 
As a data whitening method, it improves downstream classification by up to 3.9\% compared to methods that don't use whitening.
Finally, we show that \name{} can be used to visualize hyperbolic data in two dimensions.

\end{abstract}

\section{Introduction}
Learning representations of data in hyperbolic spaces has recently attracted important interest in Machine Learning (ML)~\cite{nickel2017poincare,sala2018representation} due to their ability to represent hierarchical data with high fidelity in low dimensions~\cite{sarkar2011low}. 
Many real-world datasets exhibit hierarchical structures, and hyperbolic embeddings have led to state-of-the-art results in applications such as question answering~\cite{tay2018hyperbolic}, node classification~\cite{chami2019hyperbolic}, link prediction~\cite{balazevic2019multi,chami2020low} and word embeddings~\cite{tifrea2018poincar}.
These developments motivate the need for algorithms that operate in hyperbolic spaces such as nearest neighbor search~\cite{krauthgamer2006algorithms,wu2020nearest}, hierarchical clustering~\cite{monath2019gradient,chami2020trees}, or dimensionality reduction which is the focus of this work. 

Euclidean Principal Component Analysis (PCA) is a fundamental dimensionality reduction technique which seeks directions that best explain the original data. %
PCA is an important primitive in data analysis and has many important uses such as (i) dimensionality reduction (e.g.\ for memory efficiency), (ii) data whitening and pre-processing for downstream tasks, and (iii) data visualization. 

Here, we seek a generalization of PCA to hyperbolic geometry.
Given a core notion of \textbf{directions}, PCA involves the following ingredients:
\setlist{nolistsep}
\begin{enumerate}[topsep=0ex,leftmargin=*]
    \item A nested sequence of affine \textbf{subspaces} (flags) spanned by a set of directions.
    \item A \textbf{projection} method which maps points to these subspaces while preserving information (e.g.\ dot-product) along each direction. 
    \item A \textbf{variance} objective to help choose the best directions.
\end{enumerate}

\begin{figure*}
    \centering
    \begin{subfigure}[b]{0.3250\textwidth}
    \centering
        \includegraphics[width=0.6\textwidth]{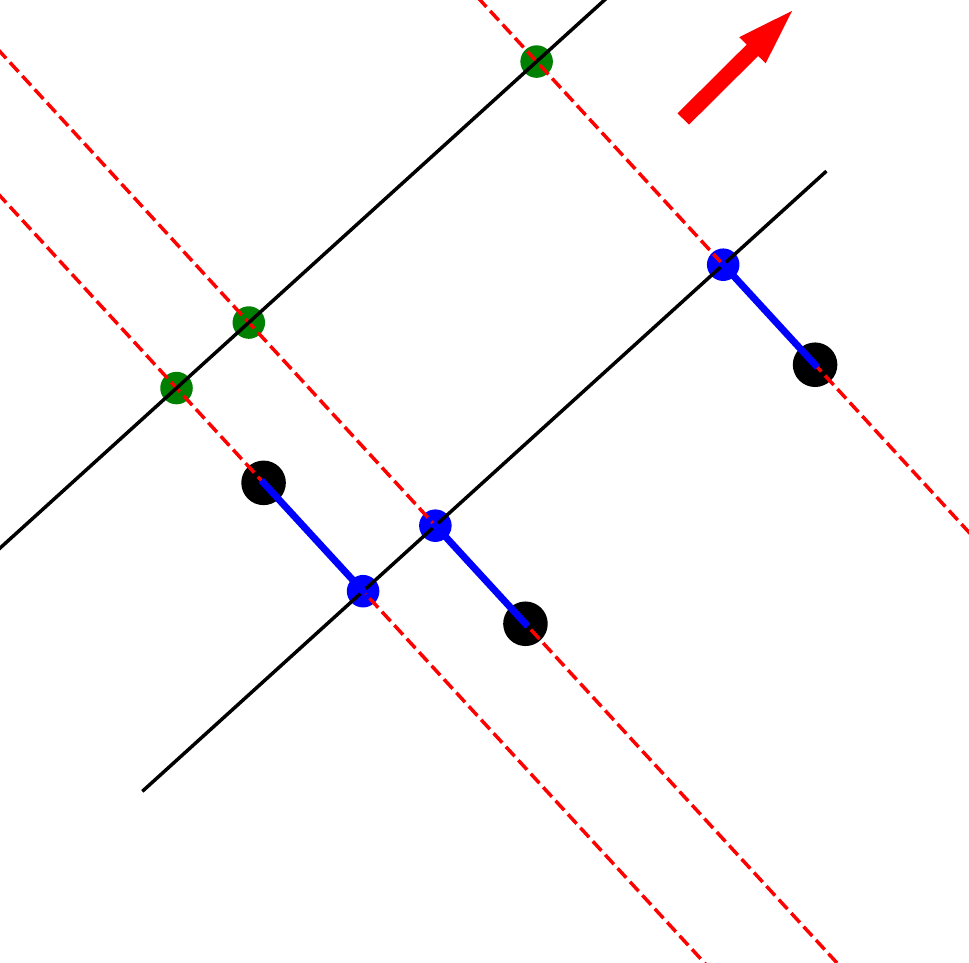}
        \caption{Euclidean projections.}
        \label{fig:euc_lines}
    \end{subfigure}
    \begin{subfigure}[b]{0.3250\textwidth}
    \centering
        \includegraphics[width=0.65\textwidth]{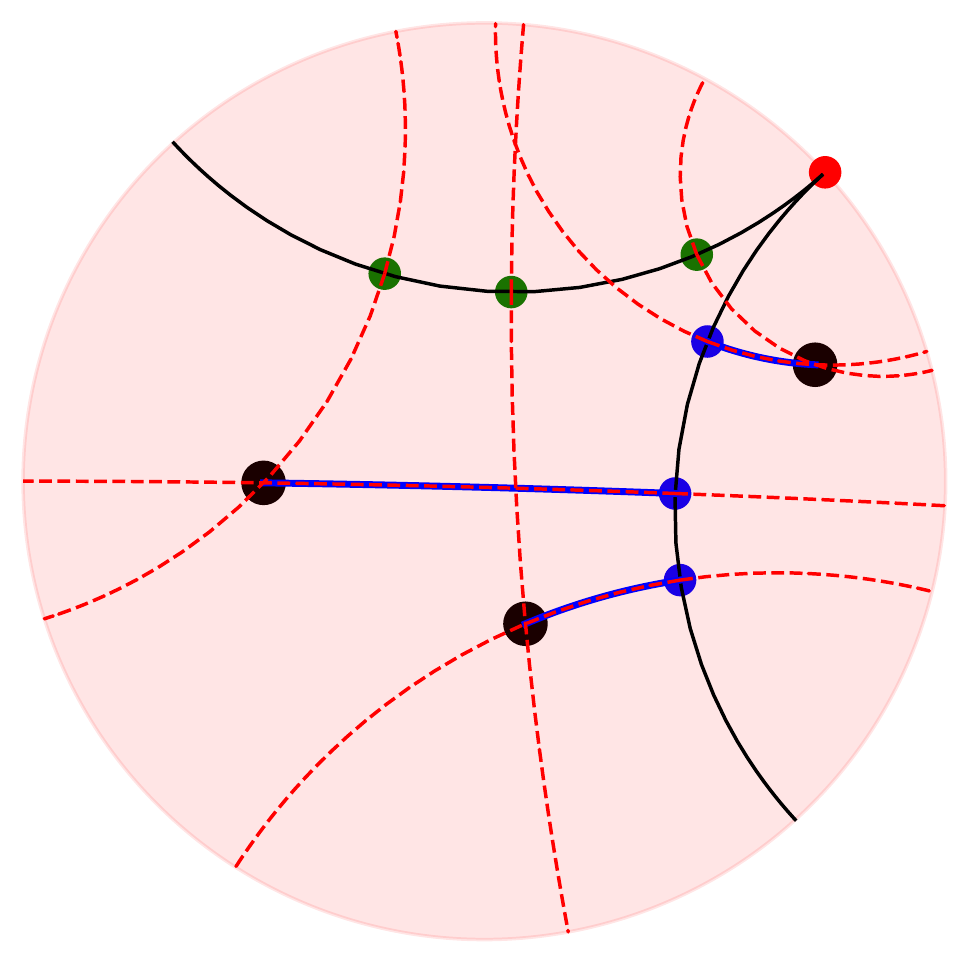}
        \caption{Hyperbolic geodesic projections.}
        \label{fig:geo_lines}
    \end{subfigure}
    \begin{subfigure}[b]{0.3250\textwidth}
    \centering
        \includegraphics[width=0.65\textwidth]{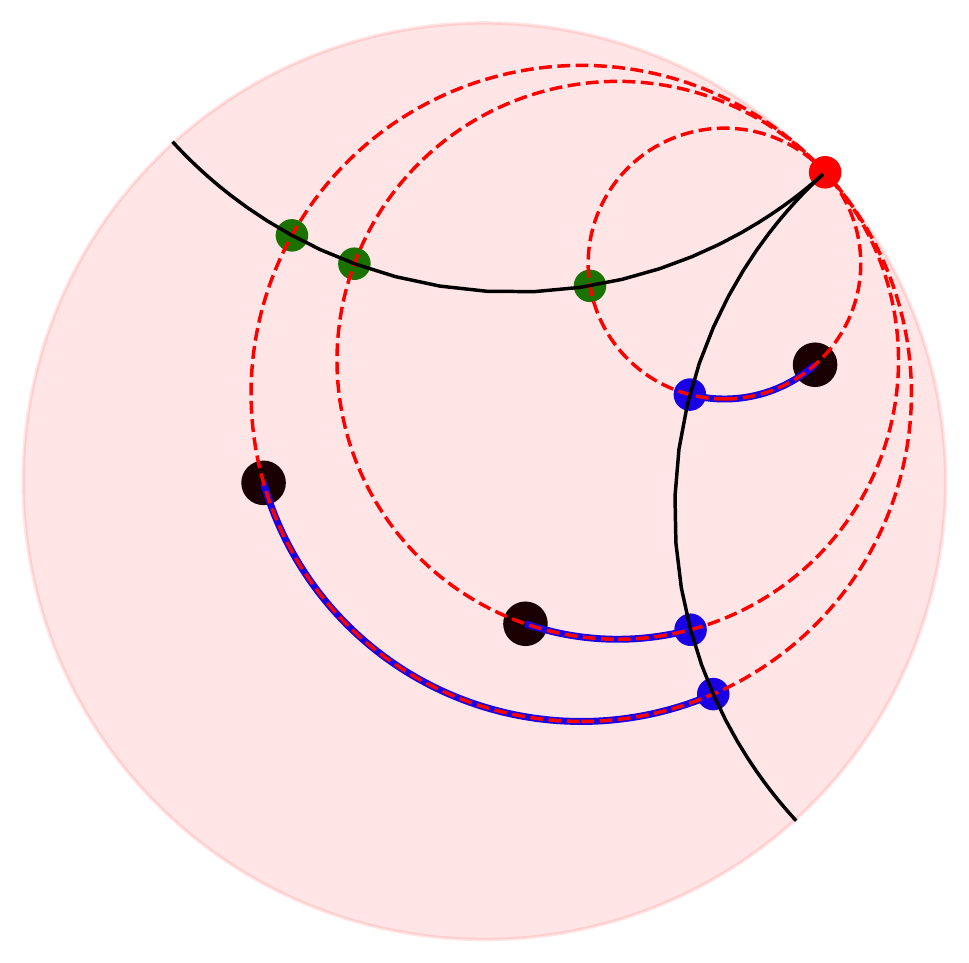}
        \caption{Hyperbolic horospherical projections.}
        \label{fig:horo_lines}
    \end{subfigure}
    \caption{
    Given datapoints (black dots), Euclidean and horospherical projections preserve distance information across different subspaces (black lines) pointing towards the same direction or point at infinity, while geodesic projections do not. 
    ($\text{a}$): Distances between points, and therefore the explained variance, are invariant to translations along orthogonal directions (red lines).
    ($\text{b}$): Geodesic projections do not preserve this property: distances between green projections are not the same as distances between blue projections.
    ($\text{c}$): 
    Horospherical projections project data by sliding it along the complementary direction defined by horospheres (red circles) and there exist an isometric mapping between the blue and green projections. 
    }\label{fig:lines}
\end{figure*}
The PCA algorithm is then defined as combining these primitives: given a dataset, it chooses directions that maximize the \emph{variance} of \emph{projections} onto a \emph{subspace} spanned by those directions, so that the resulting sequence of {directions} optimally explains the data.
Crucially, the algorithm only depends on the directions of the \emph{affine} subspaces and not their locations in space~(\cref{fig:euc_lines}). Thus, in practice we can assume that they all go through the origin (and hence become \emph{linear} subspaces), which greatly simplifies computations.

Generalizing PCA to manifolds is a challenging problem that has been studied for decades, starting with Principal Geodesic Analysis (PGA)~\cite{fletcher2004pga} which parameterizes subspaces using tangent vectors at the mean of the data, and maximizes distances from projections to the mean to find optimal directions. 
More recently, the Barycentric Subspace Analysis (BSA) method~\cite{pennec2018barycentric} was introduced.
It finds a more general parameterization of nested sequences of submanifolds by minimizing the unexplained variance.
However, both PGA and BSA map points onto submanifolds using closest-point or \emph{geodesic projections}, which do not attempt to preserve information along principal directions; for example, they cannot isometrically preserve hyperbolic distances between \emph{any} points and shrink all path lengths by an exponential factor~(\cref{prop:distortion}).

Fundamentally, all previous methods only look for \emph{subspaces} rather than directions that explain the data, and can perhaps be better understood as principal \emph{subspace} analysis rather than principal \emph{component} analysis. Like with PCA, they assume that all optimal subspaces go through a chosen \emph{base point}, but unlike in the Euclidean setting, this assumption is now unjustified: ``translating'' the submanifolds does not preserve distances between projections~(\cref{fig:geo_lines}). Furthermore, the dependence on the base point is sensitive: as noted above, the shrink factor of the projection depends exponentially on the distances between the subspaces and the data. Thus, having to choose a base point increases the number of necessary parameters and reduces stability.

Here, we propose \name{}, a dimensionality reduction method for data defined in hyperbolic spaces which better preserves the properties of Euclidean PCA.
We show how to interpret \textbf{directions} using \emph{points at infinity} (or \emph{ideal points}), which then allows us to generalize core properties of PCA:  
\setlist{nolistsep}
\begin{enumerate}[topsep=0ex,leftmargin=*]
    \item To generalize the notion of affine \textbf{subspace}, we propose parameterizing geodesic subspaces as the sets spanned by these ideal points.
    This yields multiple viable {nested subspaces} (flags)~(\cref{subsec:horo_components}). 
    \item To maximally preserve information in the original data, we propose a new \textbf{projection} method that uses horospheres, a generalization of complementary directions for hyperbolic space. 
    In contrast with geodesic projections, these projections exactly preserve information -- specifically, distances to ideal points -- along each direction. 
    Consequently, they preserve distances between points much better than geodesic projections~(\cref{subsec:horo_proj}).
    \item Finally, we introduce a simple generalization of explained \textbf{variance} that is a function of distances only and can be computed in hyperbolic space~(\cref{subsec:obj}). 
\end{enumerate}
Combining these notions, we propose an algorithm that seeks a sequence of \emph{principal components} that best explain variations in hyperbolic data.
We show that this formulation retains the \emph{location-independence} property of PCA: translating target submanifolds along orthogonal directions (horospheres) preserves projected distances~(\cref{fig:horo_lines}). In particular, the algorithm's objective depends only on the directions and not locations of the submanifolds~(\cref{subsec:horo_isom}).

We empirically validate \name{} on real datasets and for three standard PCA applications.
First, (i) we show that it yields much lower distortion and higher explained variance than existing methods, reducing average distortion by up to 77\%.
Second, (ii) we validate that it can be used for data pre-processing, improving downstream classification by up to 3.8\% in Average Precision score compared to methods that don't use whitening.
Finally, (iii) we show that the low-dimensional representations learned by \name{} can be visualized to qualitatively interpret hyperbolic data.

\section{Background}
We first review some basic notions from hyperbolic geometry; a more in-depth treatment is available in standard texts~\cite{lee2013smooth}.
We discuss the generalization of coordinates and directions in hyperbolic space and then review geodesic projections. 
We finally describe generalizations of the notion of mean and variance to non-Euclidean spaces.

\subsection{The Poincar\'e Model of Hyperbolic Space}
Hyperbolic geometry is a Riemannian geometry with constant negative curvature $-1$, where curvature measures deviation from flat Euclidean geometry. For easier visualizations, we work with the $d$-dimensional Poincar\'e model of hyperbolic space: $\H^d=\{{x}\in\R^d: \|x \|<1\}$,
where $\|\cdot \|$ is the Euclidean norm. 
In this model, the Riemannian distance can be computed in cartesian coordinates by:
\begin{equation}\label{eq:hyp_distance}
    d_\H (x, y)=\arccosh \left(1+2\frac{\|x-y\|^2}{(1-\|x\|^2)(1-\|y\|^2)} \right).
\end{equation}
\paragraph{Geodesics} Shortest paths in hyperbolic space are called \emph{geodesics}. In the Poincar\'e model, they are represented by straight segments going through the origin and circular arcs perpendicular to the boundary of the unit ball~(\cref{fig:horo_geo}).

\paragraph{Geodesic submanifolds} A submanifold $M \subset \H^d$ is called \emph{(totally) geodesic} if for every $x, y \in M$, the geodesic line connecting $x$ and $y$ belongs to $M$. This generalizes the notion of affine subspaces in Euclidean spaces. In the Poincar\'e model, geodesic submanifolds are represented by linear subspaces going through the origin and spherical caps perpendicular to the boundary of the unit ball.

\begin{table}[t]
    \centering
    \begin{tabular}{@{}lll@{}}
        \toprule
        & Euclidean & Hyperbolic \\
        \midrule
        Component & Unit vector $w$ & Ideal point $p$ \\
        Coordinate & Dot product $x \cdot w$ & Busemann func. $B_{p}(x)$ \\
        \bottomrule
    \end{tabular}
    \caption{Analogies of components and their corresponding coordinates, in both Euclidean and hyperbolic space.}
    \label{table:coord}
\end{table}
\subsection{Directions in Hyperbolic space}\label{subsec:ideal}
The notions of directions, and coordinates in a given direction can be generalized to hyperbolic spaces as follows.

\paragraph{Ideal points} As with parallel rays in Euclidean spaces, geodesic rays in $\H^d$ that stay close to each other can be viewed as sharing a common \emph{endpoint at infinity}, also called an \emph{ideal point}. Intuitively, ideal points represent \emph{directions} along which points in $\H^d$ can move toward infinity.
The set of ideal points $\Sinf^{d-1}$, called the \emph{boundary at infinity} of $\H^d$, is represented by the unit sphere $\Sinf^{d-1}=\{ \|x\| = 1\}$ in the Poincar\'e model.
We abuse notations and say that a geodesic submanifold $M \subset \H^d$ contains an ideal point $p$ if the boundary of $M$ in $\Sinf^{d-1}$ contains $p$.

\paragraph{Busemann coordinates} In Euclidean spaces, each \emph{direction} can be represented by a unit vector $w$. The \emph{coordinate} of a point $x$ in the direction of $w$ is simply the dot product $w \cdot x$. In hyperbolic geometry, directions can be represented by ideal points but dot products are not well-defined. Instead, we take a ray-based perspective:
note that in Euclidean spaces, if we shoot a ray in the direction of $w$ from the origin, the coordinate $w \cdot x$ can be viewed as the \emph{normalized distance to infinity in the direction of that ray}. 
In other words, as a point $y=tw$, ($t>0$) moves toward infinity in the direction of $w$:
\[
w \cdot x = \lim_{t \to \infty} \left( d(0, tw) -  d(x, tw)\right).
\]

This approach generalizes to other geometries: given a unit-speed geodesic ray $\gamma(t)$, the \emph{Busemann function} $B_\gamma (x)$ of $\gamma$ is defined as:\footnote{Note that compared to the above formula, the sign convention is flipped due to historical reasons.} 
\[
B_\gamma (x) = \lim_{t \to \infty} \left( d(x, \gamma(t)) - t \right).
\]
Up to an additive constant, this function only depends on the endpoint at infinity of the geodesic ray, and not the starting point $\gamma(0)$. Thus, given an ideal point $p$, we define the Busemann function $B_p(x)$ of $p$ to be the Busemann function of the geodesic ray that starts from the origin of the unit ball model and has endpoint $p$. Intuitively, it represents the \emph{coordinates} of $x$ in the direction of $p$.
In the Poincar\'e model, there is a closed formula:
\[
    B_{p}(x) = \ln \frac{ \| p - x \|^2 }{1 - \| x\|^2}.
\]

\paragraph{Horospheres} 
The level sets of Busemann functions $B_p (x)$ are called \emph{horospheres centered at} $p$. In this sense, they resemble spheres, which are level sets of distance functions. However, intrinsically as Riemannian manifolds, horospheres have curvature zero and thus also exhibit many properties of planes in Euclidean spaces.

Every geodesic with endpoint $p$ is orthogonal to every horosphere centered at $p$. Given two horospheres with the same center, every orthogonal geodesic segment connecting them has the same length. 
In this sense, \textit{concentric horospheres resemble parallel planes in Euclidean spaces}.
In the Poincar\'e model, horospheres are Euclidean spheres that touch the boundary sphere $\Sinf^{d-1}$ at their ideal centers~(\cref{fig:horo_geo}). Given an ideal point $p$ and a point $x$ in $\H^d$, there is a unique horosphere $S(p, x)$ passing through $x$ and centered at $p$.

\subsection{Geodesic Projections}

PCA uses orthogonal projections to project data onto subspaces. Orthogonal projections are usually generalized to other geometries as \emph{closest-point projections}.
Given a target submanifold $M$, each point $x$ in the ambient space is mapped to the closest-point to it in $M$:
\[
    \piGeo_M (x) = \argmin_{y \in M} d_M(x, y).
\]
One could view $\piGeo_M (\cdot)$ as the map that pushes each point $x$ along an orthogonal geodesic until it hits $M$. For this reason, it is also called \emph{geodesic projection}.
In the Poincar\'e model, these can be computed in closed-form (see~\cref{sec:app_exp}).

\subsection{Manifold Statistics}
PCA relies on data statistics which do not generalize easily to hyperbolic geometry.
One approach to generalize the arithmetic mean is to notice that it is the minimizer of the sum of squared distances to the inputs.
Motivated by this, the Fr\'echet mean~\cite{frechet1948elements} of a set of points $S$ in a Riemannian manifold $(M, d_M)$ is defined as:
\[
  \mu_M(S)\coloneqq \argmin_{y \in M} \sum_{x \in S}d_M(x, y)^2.\label{eq:frechet_mean} 
\]
This definition only depends on the intrinsic distance of the manifold. For hyperbolic spaces, since squared distance functions are convex, $\mu(S)$ always exists and is unique.\footnote{For more general geometries, existence and uniqueness hold if the data is well-localized ~\cite{kendall1990probability}.}
Analogously, the Fr\'echet variance is defined as:

\begin{equation}
    \sigma^2_M(S)\coloneqq\frac{1}{|S|} \sum_{x\in S}^N d_M(x, \mu(S))^2.
    \label{eq:frechet_var}
\end{equation}
We refer to~\cite{huckemann2020statistical} for a discussion on different intrinsic statistics in non-Euclidean spaces, and a study of their asymptotic properties.
\begin{figure}
    \begin{small}
    \centering
    \iftoggle{arxiv}{
      \includegraphics[width=2.5in]{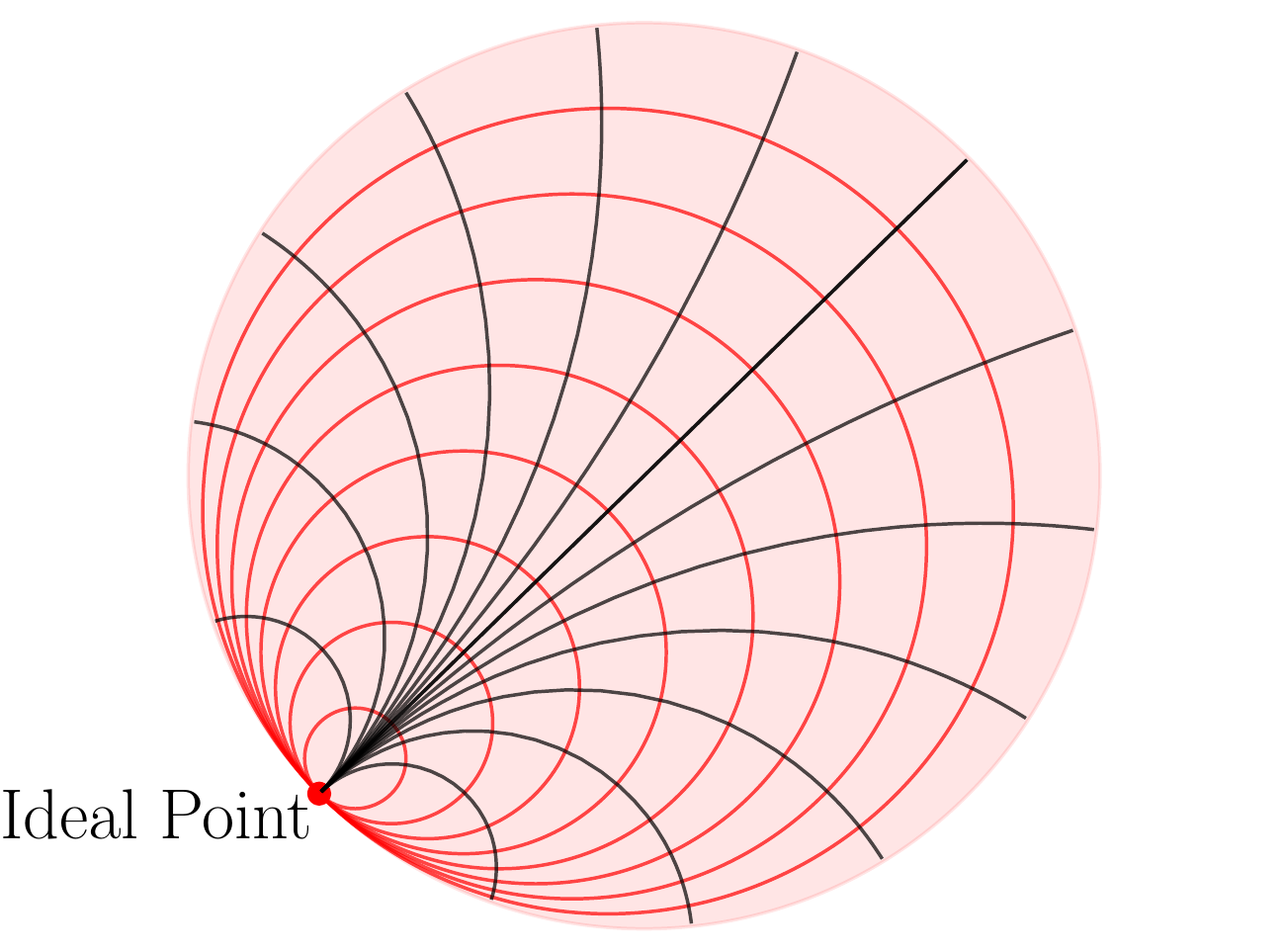}
    }{
      \includegraphics[width=0.55\columnwidth]{images/horo_geo.pdf}
    }
    \caption{Hyperbolic geodesics (black lines) going through an ideal point (in red), and horospheres (red circles) centered at that same ideal point.
    The hyperbolic lengths of geodesic segments between two horospheres are equal.}\label{fig:horo_geo}
    \end{small}
\end{figure}

\section{Generalizing PCA to the Hyperbolic Space}\label{sec:horo}
We now describe our approach to generalize PCA to hyperbolic spaces. 
The starting point of \name{} is to pick $1\le K\le d$ ideal points $p_1, \dots, p_K \in \Sinf^{d-1}$ to represent $K$ directions in hyperbolic spaces~(\cref{subsec:ideal}).
Then, we generalize the core concepts of Euclidean PCA.
In~\cref{subsec:horo_components}, we show how to generalize \emph{flags}.
In~\cref{subsec:horo_proj}, we show how to \emph{project} points onto the lower-dimensional submanifold spanned by a given set of directions, while preserving information along each direction.
In~\cref{subsec:obj}, we introducing a \emph{variance objective} to optimize and show that it is a function of the directions only. 

\subsection{Hyperbolic Flags}
\label{subsec:horo_components}
In Euclidean spaces, one can take the linear spans of more and more components to define a nested sequence of linear subspaces, called a flag. To generalize this to hyperbolic spaces, we first need to adapt the notion of linear/affine spans.
Recall that geodesic submanifolds are generalizations of affine subspaces in Euclidean spaces.
\begin{definition}
 Given a set of points $S$ (that could be inside $\H^d$ or on the boundary sphere $\Sinf^{d-1}$), the smallest geodesic submanifold of $\H^d$ that contains $S$ is called the \emph{geodesic hull} of $S$ and denoted by $\GH(S)$.
\end{definition}

Thus, given $K$ ideal points $p_1, p_2, \dots, p_K$ and a base point $b \in \H^d$, we can define a nested sequence of geodesic submanifolds $\GH(b,p_1) \subset \GH(b,p_1,p_2) \subset \dots \subset \GH(b,p_1, \dots, p_K)$. This will be our notion of flags.

\begin{remark}
The base point $b$ is only needed here for technical reasons, just like an origin $\mathbf{o}$ is needed to define linear spans in Euclidean spaces. We will see next that it does not affect the projection results or objectives~(\cref{thm:horosphere-proj-base-point-independent}). 
\end{remark}
\begin{remark}
We assume that none of $b, p_1, \dots, p_K$ are in the geodesic hull of the other $K$ points. This is analogous to being linearly independent in Euclidean spaces.
\end{remark}

\subsection{Projections via Horospheres}\label{subsec:horo_proj}
In Euclidean PCA, points are projected to the subspaces spanned by the given directions in a way that preserves coordinates in those directions. We seek a projection method in hyperbolic spaces with a similar property.

Recall that coordinates are generalized by Busemann functions~(\cref{table:coord}), and that horospheres are level sets of Busemann functions. Thus, we propose a projection that preserves coordinates by moving points along horospheres. It turns out that this projection method also preserves distances better than the traditional geodesic projection.

As a toy example, we first show how the projection is defined in the $K = 1$ case (i.e. projecting onto a geodesic) and why it tends to preserve distances well. 
We will then show how to use $K \ge 1$ ideal points simultaneously.

\subsubsection{Projecting onto \texorpdfstring{$K=1$}{K = 1} Directions}
For $K = 1$, we have one ideal point $p$ and base point $b$, and the geodesic hull $\GH(b,p)$ is just a geodesic $\gamma$. Our goal is to map every $x \in \H^d$ to a point $\piHoro_{b,p} (x)$ on $\gamma$ that has the same Busemann coordinate in the direction of $p$:
$$B_p(x) = B_p(\piHoro_{b,p} (x)).$$
Since level sets of $B_p (x)$ are horospheres centered at $p$, the above equation simply says that $\piHoro_{b,p}(x)$ belongs to the horosphere $S(p,x)$ centered at $p$ and passing through $x$. 
Thus, we define:
\begin{equation}
    \piHoro_{b,p} (x)\coloneqq \gamma \cap S(p, x).
\end{equation}
Another important property that $\piHoro_{b,p}(\cdot)$ shares with orthogonal projections in Euclidean spaces is that it preserves distances along a direction -- lengths of geodesic segments that point to $p$ are preserved after projection~(\cref{fig:horo_geo_proj}):
\begin{prop} \label{prop:horosphere-proj-preserv-best-case-1D}
For any $x \in \H^d$, if $y \in \GH(x,p)$ then:
\[
d_\H(\piHoro_{b,p}(x), \piHoro_{b,p}(y)) = d_\H(x,y).
\]
\end{prop}
\begin{proof}
This follows from the remark in \cref{subsec:ideal} about horospheres: every geodesic going through $p$ is orthogonal to every horosphere centered at $p$, and every orthogonal geodesic segment connecting concentric horospheres has the same length~(\cref{fig:horo_geo}). In this case, the segments from $x$ to $y$ and from $\piHoro_{b,p}(x)$ to $\piHoro_{b,p}(y)$ are two such segments, connecting $S(p,x)$ and $S(p,y)$.
\end{proof}
\begin{figure}
    \begin{small}
    \centering
    \iftoggle{arxiv}{
      \includegraphics[width=2in]{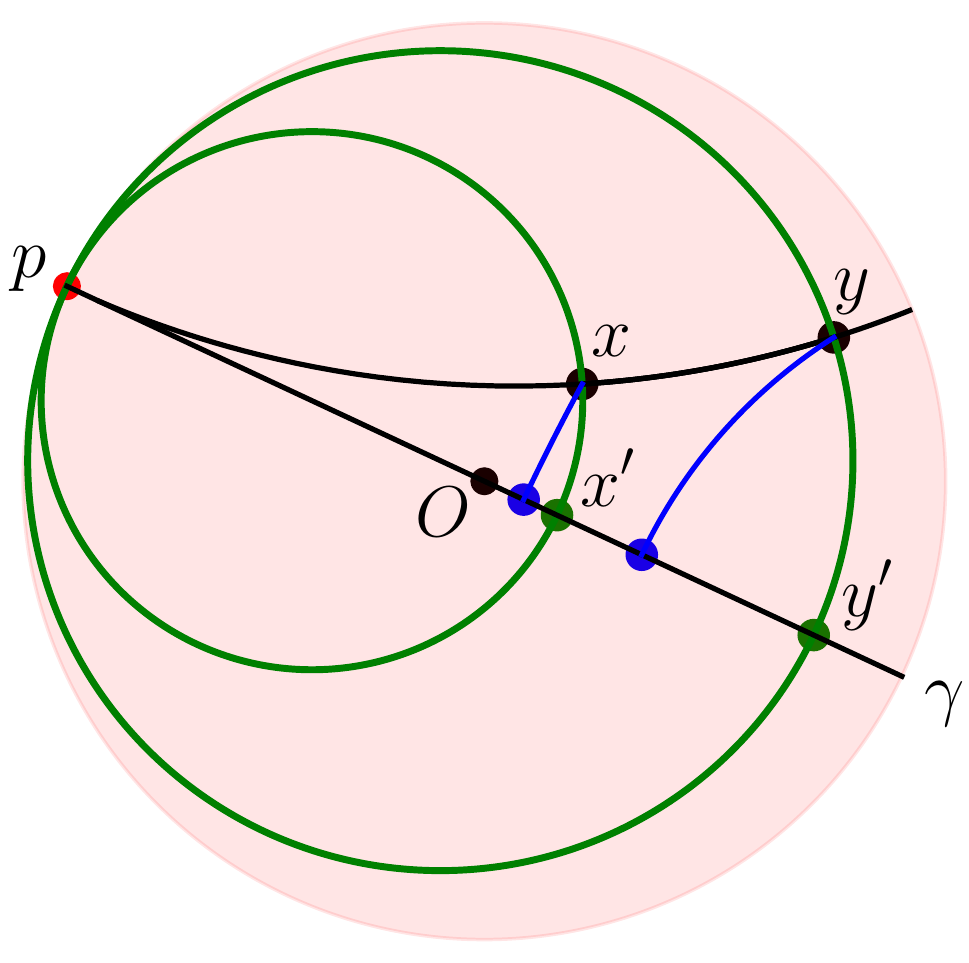}
    }
    {
      \includegraphics[width=0.42\columnwidth]{images/horo_geo_proj}
    }
    \caption{$x', y'$ are horospherical (green) projections of $x, y$. \cref{prop:horosphere-proj-preserv-best-case-1D} shows $d_\H(x',y') = d_\H (x,y)$. The distance between the two geodesic (blue) projections is smaller. 
    }\label{fig:horo_geo_proj}
    \end{small}
\end{figure}

\subsubsection{Projecting onto \texorpdfstring{$K>1$}{K > 1} Directions}\label{subsec:horo_proj_high}
We now generalize the above construction to projections onto higher-dimensional submanifolds. We describe the main ideas here; \cref{appendix-sec:horo} contains more details, including an illustration in the case $K=2$ (\cref{fig:horo_proj_3d}).

Fix a base point $b \in \H^d$ and $K > 1$ ideal points $\{p_1, \dots, p_K\}$. We want to define a map from $\H^d$ to $M \coloneqq \GH(b, p_1, \dots, p_K)$ that preserves the Busemann coordinates in the directions of $p_1, \dots, p_K$, i.e.:
\[
B_{p_j} (x)=B_{p_j}\left( \piHoro_{b,p_1,\dots,p_K}(x) \right)\text{ for every }j = 1, \dots, K.
\]

As before, the idea is to take the intersection with the horospheres centered at $p_j$'s and passing through $x$:
\begin{align*}
\pi^{\text{H}}_{b, p_1,\dots,p_K}: \H^d &\to M \\
x &\mapsto M \cap S(p_1, x) \cap \dots \cap S(p_K, x).
\end{align*}
It turns out that this intersection generally consists of two points instead of one. When that happens, one of them will be strictly closer to the base point $b$, and we define $\piHoro_{b,p_1,\dots,p_K}(x)$ to be that point.

As with~\cref{prop:horosphere-proj-preserv-best-case-1D}, $\piHoro_{b,p_1,\dots,p_K}(\cdot)$ preserves distances along $K$-dimensional manifolds %
(\cref{cor:horosphere-proj-preserve-some-distance}).
In contrast, geodesic projections in hyperbolic spaces \emph{never} preserve distances (except between points already in the target):

\begin{restatable}[]{prop}{geoshrink}
\label{prop:geodesic-proj-shrink-path}
Let $M \subset \H^d$ be a geodesic submanifold. Then every geodesic segment of distance at least $r$ from $M$ gets at least $\cosh(r)$ times shorter under the geodesic projection $\piGeo_M (\cdot)$ to $M$:
\[
\operatorname{length}(\piGeo_M (I)) \leq \frac{1}{\cosh(r)} \operatorname{length}(I).
\]\label{prop:distortion}
In particular, the shrink factor grows exponentially as the segment $I$ moves away from $M$. $\qedsymbol$
\end{restatable}
The proof is in~\cref{appendix-sec:orthogonal-projection-distortion}.

\paragraph{Computation} Interestingly, horosphere projections can be computed without actually computing the horospheres. The key idea is that if we let $P = \GH(p_1, \dots, p_K)$ be the geodesic hull of the horospheres' centers, then the intersection $S(p_1, x) \cap \dots \cap S(p_K,x)$
is simply the orbit of $x$ under the rotations around $P$. (This is true for the same reason that spheres whose centers lie on the same axis must intersect along a circle around that axis). 
Thus, $\piHoro_{b,p_1,\dots,p_K}(\cdot)$ can be viewed as the map that rotates $x$ around until it hits $M$.
It follows that it can be computed by:
\setlist{nolistsep}
\begin{enumerate}[topsep=0ex,leftmargin=*]
    \item Find the geodesic projection $c = \piGeo_P(x)$ of $x$ onto $P$.
    \item Find the geodesic $\alpha$ on $M$ that is orthogonal to $P$ at $c$.
    \item Among the two points on $\alpha$ whose distance to $c$ equals $d_\H(x, c)$, returns the one closer to $b$.
\end{enumerate}

The detailed computations and proof that this recovers horospherical projections are provided in~\cref{appendix-sec:horo}.

\subsection{Intrinsic Variance Objective}\label{subsec:obj}
In Euclidean PCA, directions are chosen to maximally preserve information from the original data. 
In particular, PCA chooses directions that maximize the Euclidean variance of projected data. To generalize this to hyperbolic geometry, we define an analog of variance that is intrinsic, i.e. dependent only on the distances between data points. As we will see in \cref{sec:horopca}, having an intrinsic objective helps make our algorithm \emph{location (or base point) independent}.

The usual notion of Euclidean variance is the squared sum of distances to the mean of the projected datapoints. Generalizing this is challenging because non-Euclidean spaces do not have a canonical choice of mean.
Previous works have generalized variance either by using the \emph{unexplained variance} or \emph{Fr\'echet variance}.
The former is the squared sum of residual distances to the projections, and thus avoids computing a mean. However, it is not intrinsic.
The latter is intrinsic \cite{fletcher2004pga} but involves finding the Fr\'echet mean, which is not necessarily a canonical notion of mean and can only be computed by gradient descent.

Our approach uses the observation that in Euclidean space:
\[
\sigma^2(S)=\frac{1}{n}\sum_{x\in S} \left\| x - \mu(S) \right \|^2 = \frac{1}{n^2} \sum_{x,y\in S} \| x-y \|^2.
\]
Thus, we propose the following generalization of variance:
\begin{equation}
    \sigma_\mathbb{H}^2(S) = \frac{1}{n^2} \sum_{x,y\in S} d_\mathbb{H}(x, y)^2.\label{eq:variance-as-sum-of-squared-distances}
\end{equation}

This function agrees with the usual variance in Euclidean space, while being a function of distances only.
Thus it is well defined in non-Euclidean space, is easily computed,
and, as we will show next, has the desired invariance due to isometry properties of horospherical projections.

\section{\name{}}\label{sec:horopca}
\cref{sec:horo} formulated several simple primitives -- including directions, flags, projections, and variance -- in a way that is generalizable to hyperbolic geometry.
We now revisit standard PCA, showing how it has a simple definition that combines these primitives using optimization.
This directly leads to the full \name{} algorithm by simply using the hyperbolic analogs of these primitives. %

\paragraph{Euclidean PCA}
Given a dataset $S$ and a target dimension $K$, Euclidean PCA greedily finds a sequence of \emph{principal components} $p_1,\dots,p_K$ that maximizes the \emph{variance} of orthogonal \emph{projections} $\piEuc_{\mathbf{o},p_1,\dots,p_k}(\cdot)$ onto the linear \footnote{Here $\mathbf{o}$ denotes the origin, and $\piEuc_{\mathbf{o},p_1,\dots,p_k}(\cdot)$ denotes the projection onto the \emph{affine} span of $\{ \mathbf{o}, p_1,\dots,p_k \}$, which is equivalent to the \emph{linear} span of $\{p_1, \dots, p_k\}$.}
\emph{subspaces} spanned by these components:
\begin{align*}
    p_1 &=
\underset{||p||=1}{\mathrm{argmax}}\ \sigma^2(\piEuc_{\mathbf{o},p}(S))\\ \text{and}\ 
p_{k+1}&=\underset{||p||=1}{\mathrm{argmax}}\ \sigma^2(\piEuc_{\mathbf{o},p_1,\dots,p_k, p}(S)).
\end{align*}
Thus, for each $1\le k\le K$, $\{p_1, \dots, p_k\}$ is the optimal set of directions of dimension $k$.

\paragraph{\name{}}
Because we have generalized the notions of flag, projection, and variance to hyperbolic geometry,
the \name{} algorithm can be defined in the same fashion.
Given a dataset $S$ in $\H^d$ and a base point $b\in\H^d$, we seek a sequence of $K$ directions that maximizes the variance of horosphere-projected data:
\begin{equation}
    \begin{split}
        p_1 &= \argmax_{p\in\Sinf^{d-1}} \sigma^2_\H(\piHoro_{b,p}(S))\\ \text{and}\ 
        p_{k+1}&=\argmax_{p\in\Sinf^{d-1}} \sigma^2_\H(\piHoro_{b,p_1,\dots,p_k, p}(S)).
    \end{split}\label{eq:horo_opt}
\end{equation}

\paragraph{Base point independence}\label{subsec:horo_isom}
Finally, we show that algorithm~\eqref{eq:horo_opt} always returns the same results regardless of the choice of a base point $b \in \H^d$. Since our variance objective only depends on the distances between projected data points, it suffices to show that these distances do not depend on $b$.

\begin{restatable}[]{theorem}{isometry} \label{thm:horosphere-proj-base-point-independent}
For any $b,b'$ and any $x, y \in \H^d$, the two projected distances $d_\H (\piHoro_{b,p_1,\dots,p_K}(x), \piHoro_{b,p_1,\dots,p_K}(y))$ and 
$d_\H (\piHoro_{b',p_1,\dots,p_K}(x), \piHoro_{b',p_1,\dots,p_K}(y))$ are equal.
\end{restatable}

The proof is included in~\cref{appendix-sec:horo}. Thus, \name{} retains the \emph{location-independence} property of Euclidean PCA: only the \emph{directions} of target subspaces matter; their exact locations do not~(\cref{fig:lines}). Therefore, just like in the Euclidean setting, we can assume without loss of generality that $b$ is the origin $\mathbf{o}$ of the Poincar\'e model. 
This:
\setlist{nolistsep}
\begin{enumerate}[topsep=0ex,leftmargin=*]
    \item alleviates the need to use $d$ extra parameters to search for an appropriate base point, and
    \item simplifies computations, since in the Poincar\'e model, geodesics submanifolds that go through the origin are simply linear subspaces, which are easier to deal with.
\end{enumerate}
 
After computing the principal directions which span the target $M=\GH(\mathbf{o},p_1,\dots,p_K)$, the reduced dimensionality data can be found by applying an Euclidean rotation that sends $M$ to $\H^K$, which also preserves hyperbolic distances.

\section{Experiments}

\begin{table*}
\vskip 0.15in
\begin{center}
\begin{small}
\resizebox{\textwidth}{!}{\renewcommand{\arraystretch}{0.95}
  \begin{tabular}{lcccccccc}
    \toprule
    & \multicolumn{2}{c}{\textbf{Balanced Tree}} & \multicolumn{2}{c}{\textbf{Phylo Tree}}
    & \multicolumn{2}{c}{\textbf{Diseases}}  & \multicolumn{2}{c}{\textbf{CS Ph.D.}} \\
    \midrule
    & distortion ($\downarrow$) & variance ($\uparrow$) & distortion ($\downarrow$) & variance ($\uparrow$) & distortion ($\downarrow$) & variance ($\uparrow$) & distortion ($\downarrow$) & variance ($\uparrow$) \\
    \cmidrule(r){2-9}
    PCA & 0.84 & 0.34 & 0.94 & 0.40& 0.90&  0.26  & 0.84 &1.68 \\
    tPCA & 0.70 & 1.16 & 0.63 & 14.34 & 0.63 &  3.92 & 0.56 & 11.09\\
    PGA & 0.63 $\pm$ 0.07 &  2.11 $\pm$ 0.47 & 0.64 $\pm$ 0.01 & 15.29 $\pm$ 0.51 & 0.66 $\pm$ 0.02 & 3.16 $\pm$ 0.39& 0.73 $\pm$ 0.02 & 6.14 $\pm$ 0.60  \\
    PGA-Noise & 0.87 $\pm$ 0.08 & 0.29 $\pm$ 0.30 & 0.64 $\pm$ 0.02 & 15.08 $\pm$ 0.77 & 0.88 $\pm$ 0.04 & 0.53 $\pm$ 0.19 & 0.79 $\pm$ 0.03 & 4.58 $\pm$ 0.64 \\
    BSA & 0.50 $\pm$ 0.00 & 3.02 $\pm$ 0.01& {0.61} $\pm$ 0.03 &18.60 $\pm$ 1.16 & 0.52 $\pm$ 0.02 &5.95 $\pm$ 0.25 & {0.70} $\pm$ 0.01 &8.15 $\pm$ 0.96 \\
    BSA-Noise & 0.74 $\pm$ 0.12 & 1.06 $\pm$ 0.67 & 0.68 $\pm$ 0.02 & 13.71 $\pm$ 0.72 & 0.80 $\pm$ 0.11 & 1.62 $\pm$ 1.30 & 0.79 $\pm$ 0.02 & 4.41 $\pm$ 0.59 \\
    hAE & 0.26 $\pm$ 0.00 & 6.91 $\pm$ 0.00 & \underline{0.32} $\pm$ 0.04 & \underline{45.87} $\pm$ 3.52 & \underline{0.18} $\pm$ 0.00 & 14.23 $\pm$ 0.06 & \underline{0.37} $\pm$ 0.02 & \underline{22.12} $\pm$ 2.47 \\
    hMDS & \underline{0.22}& \textbf{7.54} & 0.74 &{40.51} & {0.21 } &\underline{15.05} & 0.83& {19.93}  \\
    \name{} & \textbf{0.19} $\pm$ 0.00 & \underline{7.15} $\pm$ 0.00 & \textbf{0.13} $\pm$ 0.01 & \textbf{69.16} $\pm$ 1.96 & \textbf{0.15} $\pm$ 0.01 & \textbf{15.46} $\pm$ 0.19 & \textbf{0.16} $\pm$ 0.02 & \textbf{36.79} $\pm$ 0.70 \\
    \bottomrule
    \end{tabular}
    }
    \captionof{table}{Dimensionality reduction results on 10-dimensional hyperbolic embeddings reduced to two dimensions. 
    Results are averaged over 5 runs for non-deterministic methods. Best in \textbf{bold} and second best \underline{underlined}.  
    }\label{tab:pytorch_reduction}
\end{small}
\end{center}
\vskip -0.1in
\end{table*}

We now validate the empirical benefits of \name{} on three PCA uses.
First, for dimensionality reduction, \name{} preserves information (distances and variance) better than previous methods which are sensitive to base point choices and distort distances more~(\cref{subsec:exp_reduce}).
Next, we validate that our notion of hyperbolic coordinates captures variation in the data and can be used for whitening in classification tasks~(\cref{subsec:exp_cls}).
Finally, we visualize the representations learned by \name{} in two dimensions~(\cref{subsec:exp_viz}).

\subsection{Experimental Setup}\label{subsec:exp_setup}
\paragraph{Baselines}
We compare \name{} to several dimensionality reduction methods, including:
\begin{enumerate*}[label=(\arabic*)]
    \item Euclidean PCA, which should perform poorly on hyperbolic data,
    \item Exact PGA,
    \item Tangent PCA (tPCA), which approximates PGA by moving the data in the tangent space of the Fr\'echet mean and then solves Euclidean PCA,
    \item BSA,
    \item Hyperbolic Multi-dimensional Scaling (hMDS)~\cite{sala2018representation}, which takes a distance matrix as input and recovers a configuration of points that best approximates these distances,
    \item Hyperbolic autoencoder (hAE) trained with gradient descent~\cite{ganea2018hyperbolic,hinton2006reducing}.
\end{enumerate*}
To demonstrate their dependence on base points, we also include two baselines that perturb the base point in PGA and BSA.
We open-source our implementation\footnote{\url{https://github.com/HazyResearch/HoroPCA}} and refer to~\cref{sec:app_exp} for implementation details on how we implemented all baselines and \name{}. 

\paragraph{Datasets}
For dimensionality reduction experiments, we consider standard hierarchical datasets previously used to evaluate the benefits of hyperbolic embeddings. 
More specifically, we use the datasets in~\cite{sala2018representation} including a fully balanced tree, a phylogenetic tree, a biological graph comprising of diseases' relationships and a graph of Computer Science (CS) Ph.D. advisor-advisee relationships. 
These datasets have respectively 40, 344, 516 and 1025 nodes, and we use the code from~\cite{gu2018learning} to embed them in the Poincar\'e ball.
For data whitening experiments, we reproduce the experimental setup from~\cite{cho2019large} and use the Polbooks, Football and Polblogs datasets which have 105, 115 and 1224 nodes each.
These real-world networks are embedded in two-dimensions using~\citet{chamberlain2017neural}'s embedding method.

\paragraph{Evaluation metrics}
To measure distance-preservation after projection, we use average distortion. 
If $\pi(\cdot)$ denotes a mapping from high- to low-dimensional representations, the average distortion of a dataset $S$ is computed as:
\[
    \frac{1}{\binom{|S|}{2}} \sum_{x\neq y\in S}\frac{|d_\H(\pi(x), \pi(y))-d_\H(x, y)|}{d_\H(x, y)}.
\]
We also measure the Fr\'echet variance in~\cref{eq:frechet_var}, which is the analogue of the objective that Euclidean PCA optimizes\footnote{All mentioned PCA methods, including \name{}, optimize for some forms of variance but \emph{not} Fr\'echet variance or distortion.}. 
Note that the mean in~\cref{eq:frechet_var} cannot be computed in closed-form and we therefore compute it with gradient-descent.

\subsection{Dimensionality Reduction}\label{subsec:exp_reduce}
We report metrics for the reduction of 10-dimensional embeddings to two dimensions in~\cref{tab:pytorch_reduction}, and refer to~\cref{sec:app_exp} for additional results, such as more component and dimension configurations.
All results suggest that \name{} better preserves information contained in the high-dimensional representations.  

On distance preservation, \name{} outperforms all methods with significant improvements on larger datasets. 
This supports our theoretical result that horospherical projections better preserve distances than geodesic projections. 
Furthermore, \name{} also outperforms existing methods on the explained Fr\'echet variance metric on all but one dataset. 
This suggests that our distance-based formulation of the variance~(\cref{eq:variance-as-sum-of-squared-distances}) effectively captures variations in the data. 
We also note that as expected, both PGA and BSA are sensitive to base point choices: adding Gaussian noise to the base point leads to significant drops in performance. 
In contrast, \name{} is by construction base-point independent. %

\subsection{Hyperbolic Data Whitening}\label{subsec:exp_cls}
An important use of PCA is for data whitening, as it allows practitioners to remove noise and decorrelate the data, which can improve downstream tasks such as regression or classification. Recall that standard PCA data whitening consists of (i) finding principal directions that explain the data, (ii) calculating the coordinates of each data point along these directions, and (iii) normalizing the coordinates for each direction (to have zero mean and unit variance).

Because of the close analogy between \name{} and Euclidean PCA, these steps can easily map to the hyperbolic case, where we (i) use \name{} to find principal directions (ideal points), (ii) calculate the Busemann coordinates along these directions, and (iii) normalize them as Euclidean coordinates. Note that this yields Euclidean representations, which allow leveraging powerful tools developed specifically for learning on Euclidean data.

We evaluate the benefit of this whitening step on a simple classification task. We compare to directly classifying the data with Euclidean Support Vector Machine (eSVM) or its hyperbolic counterpart (hSVM), and also to whitening with tPCA.
Note that most baselines in~\cref{subsec:exp_setup} are incompatible with data whitening: hMDS does not learn a transformation that can be applied to unseen test data, while methods like PGA and BSA do not naturally return Euclidean coordinates for us to normalize. To obtain another baseline, we use a logarithmic map to extract Euclidean coordinates from PGA. 
\begin{table}[t]
\vskip 0.15in
\begin{center}
\begin{small}
  \centering
  \begin{tabular}{lccc}
    \toprule
     & \textbf{Polbooks} & \textbf{Football} & \textbf{Polblogs} \\
    \midrule
    eSVM & \underline{69.9}$\pm$ 1.2 & 20.7 $\pm$ 3.0 & 92.3 $\pm$ 1.5 \\
    hSVM & 68.3 $\pm$ 0.6 & 20.9 $\pm$ 2.5 & {92.2} $\pm$ 1.6 \\ 
    tPCA+eSVM & 68.5 $\pm$ 0.9 & {21.2} $\pm$ 2.2 & \underline{92.4} $\pm$ 1.5 \\
    PGA+eSVM & 64.4 $\pm$ 4.1 & \underline{21.7} $\pm$ 2.2 & 82.3 $\pm$ 1.2\\
    \name{}+eSVM & \textbf{72.2} $\pm$ 2.8 & \textbf{25.0} $\pm$ 1.0 & \textbf{92.8} $\pm$ 0.9 \\
    \bottomrule
    \end{tabular}
    \captionof{table}{Data whitening experiments. We report classification accuracy averaged over 5 embedding configurations.
    Best in \textbf{bold} and second best \underline{underlined}. 
    }\label{tab:cls}
\end{small}
\end{center}
\vskip -0.2 in
\end{table}

We reproduce the experimental setup from~\cite{cho2019large} who split the datasets in 50\% train and 50\% test sets, run classification on 2-dimensional embeddings and average results over 5 different embedding configurations as was done in the original paper~(\cref{tab:cls}).
\footnote{Note that the results slightly differ from~\cite{cho2019large} which could be because of different implementations or data splits.}
\name{} whitening improves downstream classification on all datasets compared to eSVM and hSVM or tPCA and PGA whitening. 
This suggests that \name{} can be leveraged for hyperbolic data whitening. 
Further, this confirms that Busemann coordinates do capture variations in the original data. 

\subsection{Visualizations}\label{subsec:exp_viz}
When learning embeddings for ML applications (e.g. classification), increasing the dimensionality can significantly improve the embeddings' quality. 
To effectively work with these higher-dimensional embeddings, it is useful to visualize their structure and organization, which often requires reducing their representations to two or three dimensions.
Here, we consider embeddings of the mammals subtree of the Wordnet noun hierarchy learned with the algorithm from~\cite{nickel2017poincare}. 
We reduce embeddings to two dimensions using PGA and \name{} and show the results in~\cref{fig:poincare_wordnet}. We also include more visualizations for PCA and BSA in~\cref{fig:poincare_wordnet_full} in the Appendix. %
As we can see, the reduced representations obtained with \name{} yield better visualizations. 
For instance, we can see some hierarchical patterns such as ``feline hypernym of cat" or ``cat hypernym of burmese cat". 
These patterns are harder to visualize for other methods, since these do not preserve distances as well as \name{}, e.g.\ PGA has 0.534 average distortion on this dataset compared to 0.078 for \name{}. 

    \begin{figure*}[t]
        \centering
        \begin{subfigure}[b]{0.49\textwidth}  
            \centering 
            \includegraphics[width=\textwidth]{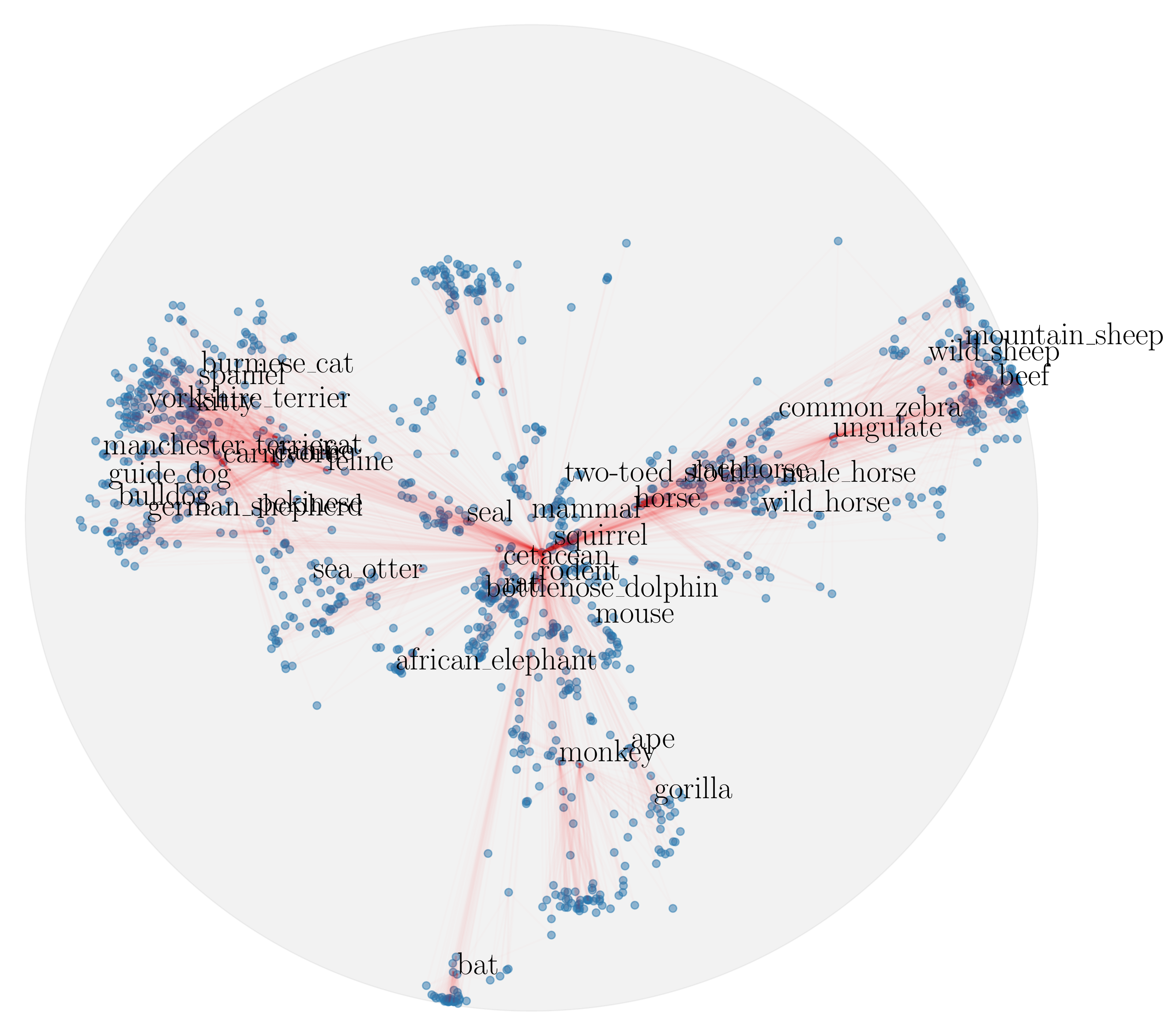}
            \caption[]%
            {{\small PGA (average distortion: 0.534)}}  
        \end{subfigure}
        \begin{subfigure}[b]{0.49\textwidth}   
            \centering 
            \includegraphics[width=\textwidth]{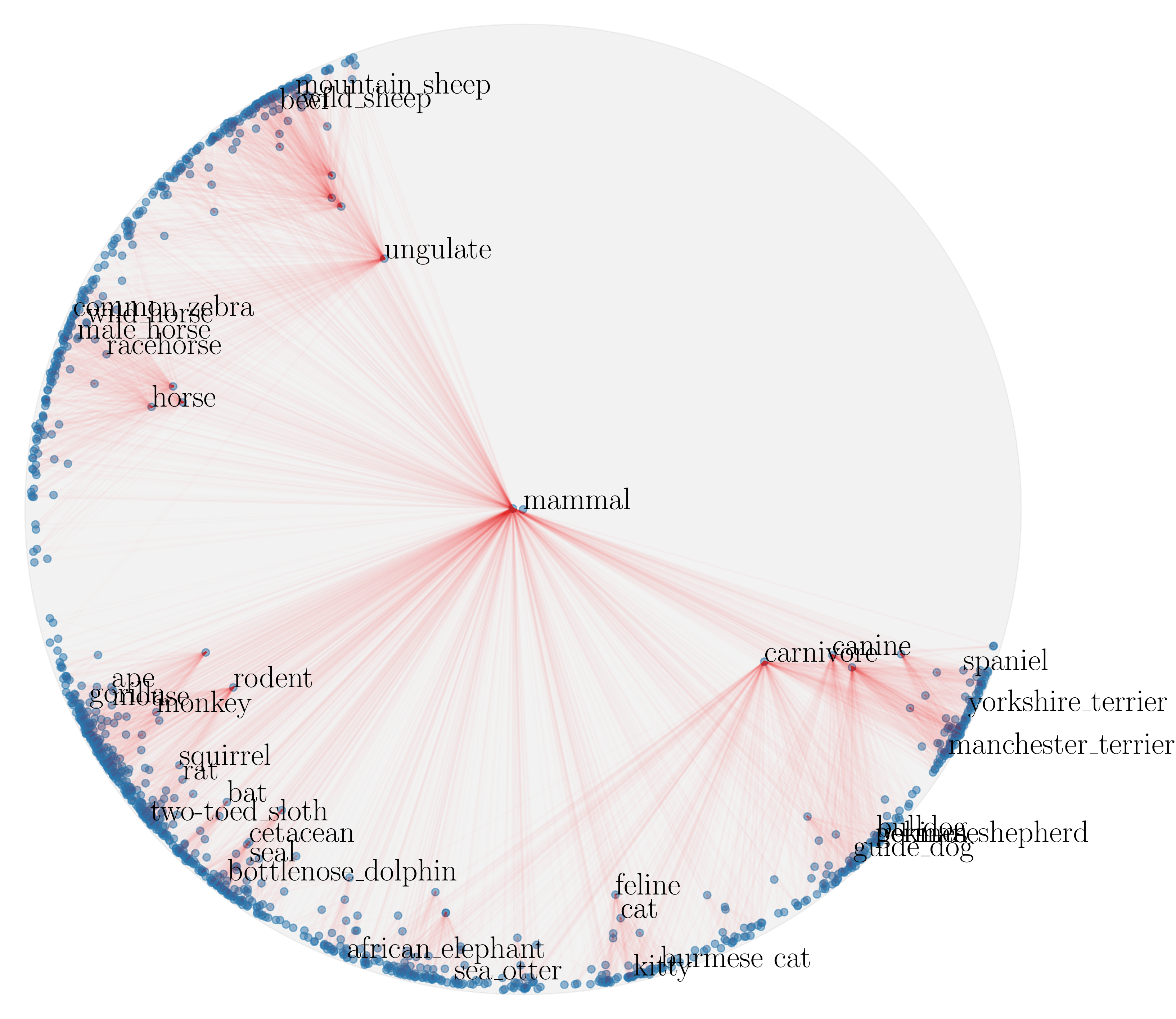}
            \caption[]%
            {{\small \name{} (average distortion: 0.078)}}     
        \end{subfigure}
        \caption[]
        {\small Visualization of embeddings of the WordNet mammal subtree computed by reducing 10-dimensional Poincar\'e embeddings~\cite{nickel2017poincare}.} 
        \label{fig:poincare_wordnet}
    \end{figure*}

\section{Related Work}\label{sec:related}
\paragraph{PCA methods in Riemannian manifolds}
We first review some approaches to extending PCA to general Riemannian geometries, of which hyperbolic geometry is a special case. For a more detailed discussion, see~\citet{pennec2018barycentric}.
The simplest such approach is tangent PCA (tPCA), which maps the data to the tangent space at the Fr\'echet mean $\mu$ using the logarithm map, then applies Euclidean PCA. A similar approach, Principal Geodesic Analysis (PGA)~\cite{fletcher2004pga}, seeks geodesic subspaces at $\mu$ that minimize the sum of squared Riemannian distances to the data. Compared to tPCA, PGA searches through the same subspaces but uses a more natural loss function.

Both PGA and tPCA project on submanifolds that go through the Fr\'echet mean. When the data is not well-centered, this may be sub-optimal, and Geodesic PCA (GPCA) was proposed to alleviate this issue~\cite{huckemann2006principal,huckemann2010intrinsic}. GPCA first finds a geodesic $\gamma$ that best fits the data, then finds other orthogonal geodesics that go through \emph{some} common point $b$ on $\gamma$. In other words, GPCA removes the constraint of PGA that $b$ is the Fr\'echet mean.
Extensions of GPCA have been proposed such as probabilistic methods~\cite{zhang2013probabilistic} and Horizontal Component Analysis~\cite{sommer2013horizontal}.

\citet{pennec2018barycentric} proposes a more symmetric approach. 
Instead of using the exponential map at a base point, it parameterizes $K$-dimensional subspaces as the \emph{barycenter} loci of $K+1$ points. 
Nested sequences of subspaces (flags) can be formed by simply adding more points.
In hyperbolic geometry, this construction coincides with the one based on geodesic hulls that we use in \cref{sec:horo}, except that it applies to points inside $\H^d$ instead of ideal points and thus needs more parameters to parameterize a flag (see \cref{remark:flag-parameter-count}).

By considering a more general type of submanifolds, \citet{hauberg2016principalcurves} gives another way to avoid the sensitive dependence on Fr\'echet mean. However, its bigger search space also makes the method computationally expensive, especially when the target dimension $K$ is bigger than $1$.

In contrast with all methods so far, \name{} relies on horospherical projections instead of geodesic projections. This yields a generalization of PCA that depends only on the directions and not specific locations of the subspaces.

\paragraph{Dimension reduction in hyperbolic geometry}
We now review some dimension reduction methods proposed specifically for hyperbolic geometry.
\citet{cvetkovski2011multidimensional} and \citet{sala2018representation} are examples of hyperbolic multidimensional scaling methods, which seek configurations of points in lower-dimensional hyperbolic spaces whose pairwise distances best approximate a given dissimilarity matrix. Unlike \name{}, they do not learn a projection that can be applied to unseen data.

\citet{tran2008dimensionality} constructs a map to lower-dimensional hyperbolic spaces whose preimages of compact sets are compact. Unlike most methods, it is data-agnostic and does not optimize any objective.

\citet{benjamini2009dimension} adapt the Euclidean Johnson–Lindenstrauss transform to hyperbolic geometry and obtain a distortion bound when the dataset size is not too big compared to the target dimension.
They do not seek an analog of Euclidean directions or projections, but nevertheless implicitly use a projection method based on pushing points along \emph{horocycles}, which shares many properties with our horo\emph{spherical} projections. In fact, the latter converges to the former as the ideal points get closer to each other.

\paragraph{From hyperbolic to Euclidean} \citet{liu2019hyperbolic} use ``distances to centroids'' to compute Euclidean representations of hyperbolic data. The Busemann functions we use bear resemblances to these centroid-based functions but are better analogs of coordinates along given directions, which is a central concept in PCA, and have better regularity properties \cite{busemann1955}. Recent works have also used Busemann functions for hyperbolic prototype learning~\cite{keller2020theory, wang2021laplacian}. These works do not define projections to lower-dimensional hyperbolic spaces. In contrast, \name{} naturally returns both hyperbolic representations (via horospherical projections) and Euclidean representations (via Busemann coordinates). This allows leveraging techniques in both settings.

\section{Conclusion}
We proposed \name{}, a method to generalize PCA to hyperbolic spaces. 
In contrast with previous PCA generalizations, \name{} preserves the core location-independence PCA property. 
Empirically, \name{} significantly outperforms previous methods on the reduction of hyperbolic data.
Future extensions of this work include deriving a closed-form solution, analyzing the stability properties of \name{}, or using the concepts introduced in this work to derive efficient nearest neighbor search algorithms or neural network operations. 

\section*{Acknowledgements}
We gratefully acknowledge the support of NIH under No. U54EB020405 (Mobilize), NSF under Nos. CCF1763315 (Beyond Sparsity), CCF1563078 (Volume to Velocity), and 1937301 (RTML); ONR under No. N000141712266 (Unifying Weak Supervision); the Moore Foundation, NXP, Xilinx, LETI-CEA, Intel, IBM, Microsoft, NEC, Toshiba, TSMC, ARM, Hitachi, BASF, Accenture, Ericsson, Qualcomm, Analog Devices, the Okawa Foundation, American Family Insurance, Google Cloud, Swiss Re, Total, the HAI-AWS Cloud Credits for Research program, the Stanford Data Science Initiative (SDSI), and members of the Stanford DAWN project: Facebook, Google, and VMWare. The Mobilize Center is a Biomedical Technology Resource Center, funded by the NIH National Institute of Biomedical Imaging and Bioengineering through Grant P41EB027060. The U.S. Government is authorized to reproduce and distribute reprints for Governmental purposes notwithstanding any copyright notation thereon. Any opinions, findings, and conclusions or recommendations expressed in this material are those of the authors and do not necessarily reflect the views, policies, or endorsements, either expressed or implied, of NIH, ONR, or the U.S. Government.

\bibliographystyle{icml2021/icml2021}
\bibliography{ref}

\clearpage
\appendix
\section{Horospherical Projection: Proofs and Discussions}\label{appendix-sec:horo}
We show that the horospherical projection $\piHoro_{b,p_1, \dots, p_K}$ in \cref{subsec:horo_proj_high} is well-defined, shares many nice properties with Euclidean orthogonal projections, and has a closed-form expression.

More specifically, this section is organized as follows. In \cref{appendix-subsec:horosphere-proj-well-defined}, we show that horospherical projections are well-defined (\cref{thm:horosphere-proj-well-defined}). In \cref{appendix-subsec:horosphere-proj-properties}, we show that they are base-point independent (\cref{cor:horosphere-proj-base-point-independent}), non-expanding (\cref{cor:horosphere-proj-non-expanding}), and distance-preserving along a family of $K$-dimensional submanifolds (\cref{cor:horosphere-proj-preserve-some-distance}). Finally, in \cref{appendix-subsec:horosphere-proj-algo}, we explain how they can be computed efficiently in the hyperboloid model.

For an illustration in the case of $K=2$ ideal points in $\H^3$, see \cref{fig:horo_proj_3d}. This figure might help explain the intuitions behind the theorems in \cref{appendix-subsec:horosphere-proj-well-defined} and \cref{appendix-subsec:horosphere-proj-properties}.

\subsection{Well-definedness} \label{appendix-subsec:horosphere-proj-well-defined}

Recall that given a base point $b \in \H^d$ and $K > 1$ ideal points $\{p_1,  \dots, p_K\}$, we would like to define $\piHoro_{b, p_1, \dots, p_K}$ by
$$x \to M \cap S(p_1, x) \cap S(p_2, x) \cap \dots \cap S(p_K, x),$$
where $M = \GH(b, p_1, \dots, p_K)$ is the target submanifold and $S(p_j, x)$ is the horosphere centered at $p_j$ and passing through $x$. For this definition to make sense, the intersection in the right hand side must contain exactly one point for each $x \in \H^d$. Unfortunately, this is not the case: In fact, the intersection generally consists of two points. Nevertheless, we will show that there is a consistent way to choose one from these two points, making the function $\piHoro_{b,p_1,\dots,p_K}$ well-defined. This is the result of \cref{thm:horosphere-proj-well-defined}.

First, to understand the above intersection, we give a more concrete description of $\cap_j S(p_j, x)$.

\begin{lemma} \label{lemma:intersection-of-horospheres-is-orbit}
    Let $P = \GH(p_1, p_2, \dots, p_K)$. Then for every $x \in \H^d$, the intersection of horospheres
	$$S(x) = S(p_1, x) \cap S(p_2, x) \cap \dots \cap S(p_k, x)$$
	is precisely the orbit of $x$ under the group $G$ of rotations around $P$.
\end{lemma}
\begin{proof} First, note that every rotation around $P$ preserves the horospheres $S(p_j, x)$ - just like how every rotation around an axis preserves every sphere whose center is on that axis. It follows that $S(x)$ is preserved by $G$. In particular, the orbit of $x$ under $G$ is contained in $S(x)$.

It remains to show that $S(x)$ contains no other points. To this end, consider any $y \neq x$ in $S(x)$. The perpendicular bisector $B$ of $x$ and $y$ is a totally geodesic hyperplane of $\H^d$ that contains every $p_j$ (because each $p_j$ is intuitively the center of a sphere that goes through $x$ and $y$). Thus, by the definition of geodesic hull, $B \supset P$. In particular, the reflection through $B$ sends $x$ to $y$ and fixes every point in $P$.
	
Now take any geodesic hyperplane $A$ that contains both $P$ and $y$, so that the reflection through $A$ fixes $y$ and every point in $P$. Then the composition of the reflections through $B$ and $A$ is a rotation that sends $x$ to $y$ and fixes every point in $P$. In other words, it is a rotation around $P$ that sends $x$ to $y$. Therefore, $y$ belongs to the orbit of $x$ under $G$.
\end{proof}

\begin{cor} \label{cor:intersection-of-horospheres-is-sphere}
    If $x \in P$ then $S(x) = \{x\}$. Otherwise, let $\piGeo_P (x)$ be the geodesic projection of $x$ onto $P$, and $Q(x)$ be the geodesic submanifold that orthogonally complements $P$ at $\piGeo_P(x)$. Then $S(x) \subset Q(x)$ and is precisely the (hyper)sphere in $Q(x)$ that is centered at $\piGeo_P(x)$ and passing through $x$.
\end{cor}
\begin{proof}
    This follows from \cref{lemma:intersection-of-horospheres-is-orbit}. If $x \in P$ then every rotation around $P$ fixes $x$, so the orbit of $x$ is just itself.
    
    Now consider the case $x \not\in P$. All rotations around $P$ must preserve $\piGeo_P (x)$ and the orthogonal complement $Q(x)$ of $P$  at $\piGeo_P (x)$. Furthermore, when restricted to the space $Q(x)$, these rotations are precisely the rotations in $Q(x)$ around the point $\piGeo_P (x)$. Thus, for every $y \in Q(x)$, the orbit of $y$ under $G$ is a sphere in $Q(x)$ centered at $\piGeo_P (x)$. In particular, $S(x)$, which is the orbit of $x$, is the sphere in $Q(x)$ that is centered at $\piGeo_P(x)$ and passing through $x$.
\end{proof}

\cref{cor:intersection-of-horospheres-is-sphere} gives the following characterization of the intersection
$$M \cap S(p_1, x) \cap S(p_2, x) \cap \dots \cap S(p_K,x) = M \cap S(x):$$
Note that $P = \GH(p_1, \dots, p_K)$ is a geodesic submanifold of $M = \GH(b, p_1, \dots, p_K)$ and that $\dim P = \dim M - 1$. Thus, through every point $y \in P$, there is a unique geodesic $\alpha$ on $M$ that goes through $y$ and is perpendicular to $P$.
\begin{cor} \label{cor:horosphere-proj-two-candidates}
    If $x \in P$ then $M \cap S(x) = \{x\}$. Otherwise, let $\alpha$ be the geodesic on $M$ that goes through $\piGeo_P (x)$ and is perpendicular to $P$. Then $M \cap S(x)$ consists of the two points on $\alpha$ whose distance to $\piGeo_P(x)$ equals $d_\H (x, \piGeo_P(x))$.
\end{cor}
\begin{proof} 
    The case $x \in P$ is clear since $x \in M$ and $S(x) = \{x\}$ by \cref{cor:intersection-of-horospheres-is-sphere}.
    
    For the other case, let $Q(x)$ be the orthogonal complement of $P$ at $\piGeo_P (x)$. Then by \cref{cor:intersection-of-horospheres-is-sphere}, $S(x)$ is precisely the sphere in $Q(x)$ that is centered at $c(x)$ and passing through $x$.
	
	Now note that $M \cap  Q(x) = \alpha$. Since $S(x) \subset Q(x)$, this gives $M \cap S(x) = M \cap Q(x) \cap S(x) = \alpha \cap S(x)$. We know that every sphere intersects every line through the center at two points.
\end{proof}

Therefore, to define $\piHoro_{b,p_1,\dots,p_K}(x)$, we just need to choose one of the two poins in $M \cap S(x)$ in a consistent way (so that the map is differentiable). To this end, note that $P$ cuts $M$ into two half-spaces, and exactly one of them contains the base point $b$. (Recall that $b \in M$ and $b \not\in P$ by the ``independence'' condition). We denote this half by $P_b$. Then, while $S(x)$ contains two points in $M$, it only contains one point in $P_b$:

\begin{theorem} \label{thm:horosphere-proj-well-defined}
    Let $\alpha$ be the geodesic on $M$ that goes through $\piGeo_P(x)$ and is perpendicular to $P$. Let $\alpha^+$ be the half of $\alpha$ contained in $P_b$. Then $\alpha^+$ intersects the sphere $S(x)$ at a unique point $x'$, which is also the unique intersection point between $P_b$ and $S(x)$. Thus, we can define $\piHoro_{b,p_1, \dots, p_K}$ by
    $$\piHoro_{b,p_1,\dots,p_K} (x) = \alpha^+ \cap S(x) = P_b \cap S(x).$$
    Equivalently, $\piHoro_{b,p_1, \dots, p_K}(x)$ is the point in $M \cap S(x)$ that is strictly closer to $b$.
\end{theorem}
\begin{proof}
    By \cref{cor:intersection-of-horospheres-is-sphere}, $S(x)$ is a sphere centered at $\piGeo_P(x)$. Then, since $\alpha^+$ is a geodesic ray starting at $\piGeo_P(x)$, it must intersect $S(x)$ at a unique point $x'$.
    
    Next, we have
    $$P_b \cap S(x) = P_b \cap M \cap S(x) = P_b \cap \alpha \cap S(x) = \alpha^+ \cap S(x),$$
    where the first equality holds because $P_b \subset M$, the second because $M \cap S(x) = \alpha \cap S(x)$ by the proof of \cref{cor:horosphere-proj-two-candidates}, and the third because $P_b \cap \alpha = \alpha^+$. Thus, $P_b \cap S(x)$ is precisely $x'$.
    
    Finally, let $x''$ be the other point of $M \cap S(x) = \alpha \cap S(x)$. Then $P$ is the perpendicular bisector of $x'$ and $x''$ in $M$. Thus, every point on the same side of $P$ in $M$ as $x'$ (but not on the boundary $P$) is strictly closer to $x'$ than to $x''$. By definition, $b$ is one of such point. Thus, $\piHoro_{b,p_1, \dots,p_K}(x)$ is the point in $M \cap S(x)$ that is closer to $b$.
\end{proof}

\begin{figure*}
    \begin{small}
    \centering
    \includegraphics[width=0.46\textwidth]{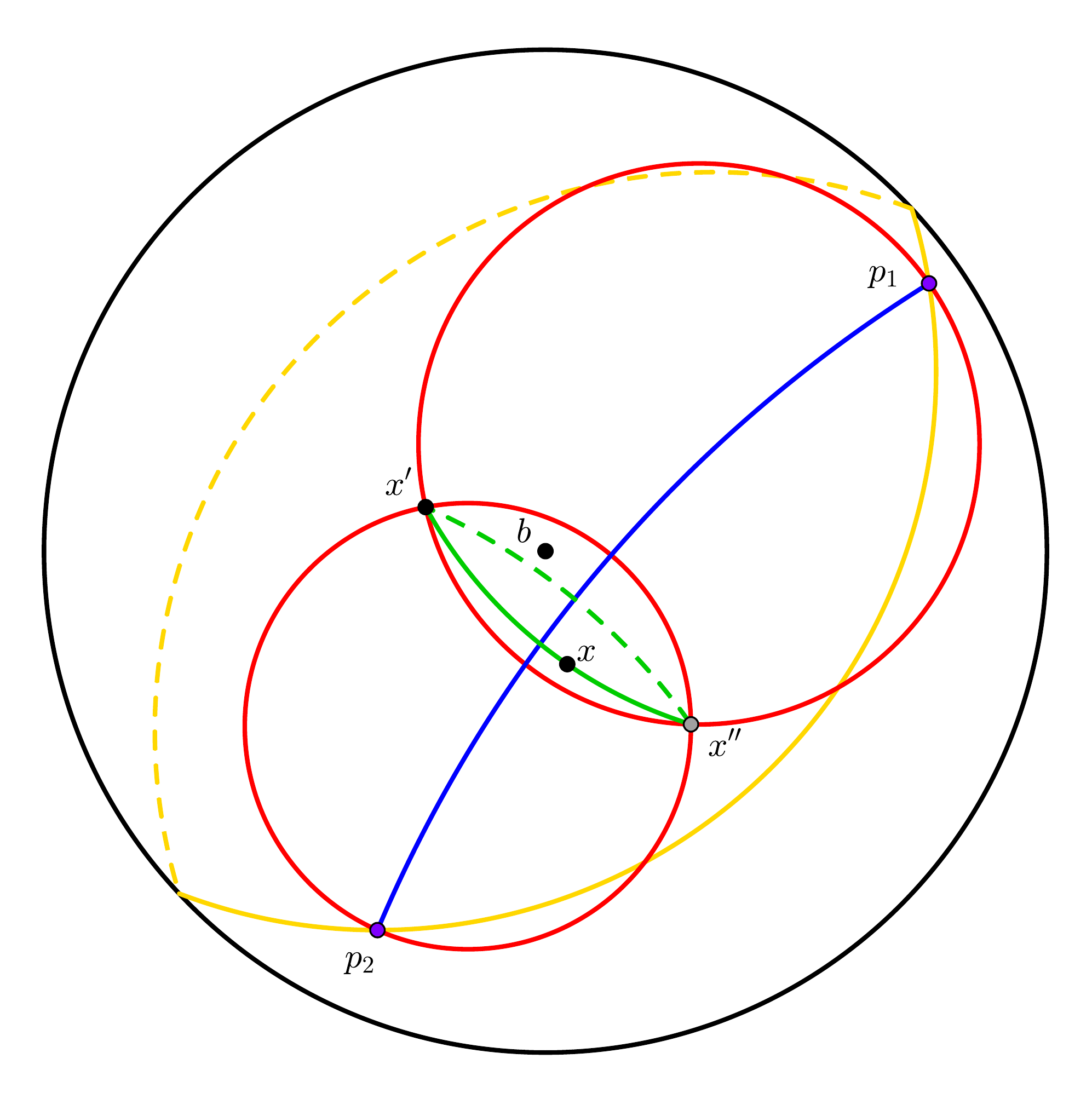}
    \caption{The horospherical projection $\piHoro_{b,p_1,p_2}$ from $\H^3$ to $2$ dimensions. Here $p_1, p_2$ are ideal points, and the base point $b \in \H^3$ is chosen to be the origin of the Poincar\'e ball. The geodesic hull $\GH(b, p_1, p_2)$ is the hyperbolic plane bounded by the {\color{yellow} yellow} circle. The geodesic between $p_1$ and $p_2$ is shown in {\color{blue} blue} and is called the spine $P = \GH(p_1,p_2)$ of this projection. For any input $x \in \H^3$, the two horospheres $S(p_1, x)$ and $S(p_2, x)$ (shown in {\color{red} red}) intersect along a circle $S(x)$ (shown in {\color{green} green}). Note that this circle is precisely the set $\{y \in \H^3: B_{p_j}(y) = B_{p_j}(x), j = 1,2\}$ and is symmetric around the spine $P$. It intersects $\GH(b,p_1,p_2)$ at two points $x'$ and $x''$, which lie on opposite sides of the spine. Since $x'$ belongs to the side containing $b$, it is closer to $b$ than $x''$ is. Thus, we define $\piHoro_{b,p_1,p_2} (x)$ to be $x'$.}\label{fig:horo_proj_3d}
    \end{small}
\end{figure*}

\subsection{Geometric properties} \label{appendix-subsec:horosphere-proj-properties}
From \cref{lemma:intersection-of-horospheres-is-orbit} and \cref{thm:horosphere-proj-well-defined}, we obtain another interpretation of $\piHoro_{b, p_1, \dots, p_K}$: It maps $\H^d$ to $P_b \subset M$ by rotating every point $x \in \H^d$ around $P$ until it hits $P_b$. In other words, we have

\begin{theorem}[The ``open book'' interpretation] \label{thm:open-book-interpretation}
	For any $x \not\in P$, let $M_x = \GH(P \cup \{x\})$. Then $P$ cuts $M_x$ into two half-spaces; we denote the half that contains $x$ by $P_x$. Then, when restricted to $P_x$, the map
    $$\piHoro_{b, p_1, \dots, p_K}: P_x \to P_b$$
    is simply a rotation around $P$. $\qedsymbol$
\end{theorem}
Following this, we call $P$ the \emph{spine} of the horosphere projection. The identity $\H^d = \cup_{x \not\in P} P_x$ can be thought of as an \emph{open book decomposition} of $\H^d$ into \emph{pages} $P_x$ that are bounded by the spine $P$. The horosphere projection $\piHoro_{b, p_1, \dots, p_K}$ then simply acts by collapsing every page onto a specified page $P_b$.

Here are some consequences of this interpretation:
\begin{cor} \label{cor:horosphere-proj-only-depends-on-spine}
	$\piHoro_{b, p_1, \dots, p_K}$ only depends on the spine $P$ and not specifically on $p_1, \dots p_K$. Thus, when we are not interested in the specific ideal points, we simply write $\piHoro_{b,P}$.
\end{cor}
\begin{proof}
    As noted above, $\piHoro_{b, p_1, \dots, p_K}(x)$ can be obtained by rotating $x$ around $P$ until it hits $P_b$. This operation does not used the exact positions of $p_j$ at all.
\end{proof}

\begin{cor} \label{cor:horosphere-proj-base-point-independent}
	The choice of $b$ does not affect the geometry of the projection $\piHoro_{b,P}$. More precisely, for any two base points $b, b' \not \in P$, the horosphere projections
	$$\piHoro_{b,P}: \H^d \to P_b \ \ \text{and} \ \  \piHoro_{b',P}: \H^d \to P_{b'}$$ only differ by a rotation $P_b \to P_{b'}$ around $P$.
\end{cor}
\begin{proof}
    By \cref{thm:open-book-interpretation}, when restricted to any page $P_x$, the maps $\piHoro_{b,P}: \H^d \to P_b$ and $\piHoro_{b',P}: \H^d \to P_{b'}$ are just rotations around $P$. The difference between any two rotations around $P$ is another rotation around $P$.
\end{proof}

In particular, \cref{cor:horosphere-proj-base-point-independent} implies \cref{thm:horosphere-proj-base-point-independent}:
\isometry*

As discussed in \cref{sec:horopca}, this theorem helps reduces parameters and simplifies the computation of $\piHoro_{b,P}$.

\begin{remark} \label{remark:flag-parameter-count}
    \cref{cor:horosphere-proj-only-depends-on-spine} and \cref{cor:horosphere-proj-base-point-independent} together imply that the \name{} algorithm \eqref{eq:horo_opt} only depends on the geodesic hulls $\GH(p_1), \GH(p_1,p_2), \dots \GH(p_1, \dots, p_K)$ of the ideal points and not the specific ideal points themselves. It follows that, theoretically, the search space of \eqref{eq:horo_opt} has dimension $dK - \frac12 K(K+1)$ -- the same as the dimension of the space of flags in Euclidean spaces.
    
    In our implementation, for simplicity we parametrize the $K$ ideal points independently, which results in a suboptimal search space dimension $(d-1)K$. Nevertheless, this is still slightly more efficient than the parametrizations used in PGA and BSA, which require have $(d+1)K$-dimensions.
\end{remark}

The following corollaries say that horospherical projections share a nice property with Euclidean orthogonal projections: When projecting to a $K$-dimensional submanifold, they preserve the distances along $K$ dimensions and collapse the distances along the other $d-K$ orthogonal dimensions:
\begin{cor} \label{cor:horosphere-proj-rotational-invariant}
    $\piHoro_{b,P}$ is invariant under rotations around $P$. In other words, if a rotation around $P$ takes $x$ to $y$ then $\piHoro_{b,P}(x) = \piHoro_{b,P}(y)$.
    
    Consequently, every $x \not\in P$ belongs to a $(d-K)$-dimensional submanifold that is collapsed to a point by $\piHoro_{b,P}$.
\end{cor}
\begin{proof}
    The open book interpretation tells us that $\piHoro_{b,P}(x)$ and $\piHoro_{b,P}(y)$ are precisely the intersections of $P_b$ with $S(x)$ and $S(y)$, respectively. If $y$ belongs to the rotation orbit $S(x)$ of $x$ then the rotation orbit $S(y)$ of $y$ is the same as $S(x)$. Thus $\piHoro_{b,P}(x) = \piHoro_{b,P}(y)$.
    
    Hence, for every $x \in \H^d$, $S(x)$ is collapsed to a point by $\piHoro_{b,P}$. To see that $\dim S(x) = d-K$ when $x \not\in P$, recall that by \cref{cor:intersection-of-horospheres-is-sphere}, if $Q(x)$ is the orthogonal complement of $P$ at $\piGeo_P (x)$ then $S(x)$ is a hypersphere inside $Q(x)$. Since the ideal points $p_j$ are assumed to be ``affinely independent,'' we have $\dim P = K-1$, so $\dim Q(x) = d - (K-1)$, and $\dim S(x) = \dim Q(x) - 1 = d - K$.
\end{proof}
\begin{cor} \label{cor:horosphere-proj-preserve-some-distance}
    For every $x \in \H^d$, there exists a $K$-dimensional totally geodesic submanifold (with boundary) that contains $x$ and is mapped isometrically to $P_b$ by $\piHoro_{b,P}$. If $x \not\in P$ then such a manifold is unique.
\end{cor}
\begin{proof}
    If $x \not\in P$ then the submanifold $P_x$ in \cref{thm:open-book-interpretation} is a geodesic submanifold that contains $x$ and is mapped isometrically to $P_b$ by $\piHoro_{b,P}$. As in the proof of \cref{cor:horosphere-proj-rotational-invariant}, we have $\dim P = K-1$ and $\dim P_x = \dim P + 1 = K$.
    
    Since \cref{cor:horosphere-proj-rotational-invariant} implies that the other $d-K$ dimensions are collapsed by $\piHoro_{b,P}$, no other distances from $x$ can be preserved. Thus, $P_x$ is the unique submanifold with the desired properties.
    
    If $x \in P$ then $x \in P_y$ for every $y \not\in P$. Note that this means \emph{every} distance from $x$ is preserved by $\piHoro_{b,P}$.
\end{proof}

The following corollaries say that, like Euclidean orthogonal projections, horospherical projections never increase distances. Thus, minimizing distortion is roughly equivalent to maximizing projected distances. This is another motivation for \cref{eq:variance-as-sum-of-squared-distances}.

\begin{cor} \label{cor:horosphere-proj-non-expanding-infinitesimal}
    For every $x \in \H^d$ and every tangent vector $\vec{v}$ at $x$,
    $$\| \piHoro_{b,P} (\vec{v}) \|_\H \leq \| \vec{v} \|_\H.$$
\end{cor}
\begin{proof}
    This follows from \cref{cor:horosphere-proj-preserve-some-distance} and \cref{cor:horosphere-proj-rotational-invariant}:
    
    If $x \in P$ then the proof of \cref{cor:horosphere-proj-preserve-some-distance} implies that $\piHoro_{b,P}$ preserves \emph{every} distance from $x$. Thus, the desired inequality is actually an equality.
    
    If $x \not \in P$ then by $\vec{v}$ has an orthogonal decomposition $\vec{v} = \vec{u} + \vec{u}^\perp$, where $\vec{u}$ and $\vec{u}^\perp$ are tangent and perpendicular to $P_x$, respectively. By
    \cref{cor:horosphere-proj-preserve-some-distance} and \cref{cor:horosphere-proj-rotational-invariant}, $\piHoro_{b,P}$ preserves the length of $\vec{u}$ while collapsing $\vec{u}^\perp$ to $0$. It follows that $\| \piHoro_{b,P} (\vec{v}) \|_\H = \| \vec{u} \|_\H \leq \| \vec{v} \|_\H$.
\end{proof}

\begin{cor} \label{cor:horosphere-proj-non-expanding}
	$\piHoro_{b,P}$ is non-expanding. In other words, for every $x, y \in \H^d$,
	$$d_\H (\piHoro_{b,P}(x), \piHoro_{b,P}(y)) \leq d_\H (x,y).$$
\end{cor}
\begin{proof}
    We first show that for any path $\gamma(t)$ in $\H^d$,
    $$\operatorname{length}(\piHoro_{b,P}(\gamma)) \leq \operatorname{length}(\gamma).$$
    Indeed, note that the velocity vector of $\piHoro_{b,P}(\gamma)$ is precisely $\piHoro_{b,P}(\dot{\gamma})$. Then, by \cref{cor:horosphere-proj-non-expanding-infinitesimal},
	\begin{align*}
	\operatorname{length}(\piHoro_{b,P}(\gamma)) &= \int \| \piHoro_{b,P} (\dot{\gamma}(t)) \|_\H dt \\
	&\leq \int \| \dot{\gamma}(t) \|_\H  dt = \operatorname{length} (\gamma).
	\end{align*}
	Now for any $x, y \in \H^d$, let $\gamma(t)$ be the geodesic segment from $x$ to $y$. Then $\piHoro_{b,P}(\gamma)$ is a path connecting $\piHoro_{b,P}(x)$ and $\piHoro_{b,P}(y)$. Thus, the projected distance is at most the length of $\piHoro_{b,P}(\gamma)$, which by the above argument is at most $\operatorname{length(\gamma)} = d_\H (x,y).$
\end{proof}

\subsection{Detour: A Review of the Hyperboloid Model} \label{appendix-subsec:hyperboloid-model-background}
\cref{thm:horosphere-proj-well-defined} suggests that $\piHoro_{b,P}(x)$ can be computed in three steps:
\begin{enumerate}
	\item Find the geodesic projection $\piGeo_P(x)$ of $x$ onto $P$.
	\item Find a geodesic ray $\alpha^+$ on $P_b$ that starts at $\piGeo_P(x)$ and is orthogonal to $P$.
	\item Return the unique point on $\alpha^+$ that is of distance $d_\H(x, \piGeo_P(x))$ from $\piGeo_P(x)$.
\end{enumerate}
It turns out that these subroutines are easier to implement in the \emph{hyperboloid model} instead of the Poincar\'e model of hyperbolic spaces. Thus, we first briefly review the basic definitions and properties of this model. The readers who are already familiar with the hyperboloid model can skip to \cref{appendix-subsec:horosphere-proj-algo}, where we describe the full algorithm.

For a more detailed treatment, see \cite{Thurston-notes}.
\begin{remark}
	The above steps are slightly different from the ones mentioned in \cref{subsec:horo_proj_high}. However, by \cref{thm:horosphere-proj-well-defined}, these two descriptions are equivalent. We decided to use the ``closer-to-$b$'' description in \cref{subsec:horo_proj_high} because it is slightly more self-contained, but the actual implementation will be based on the above three steps.
\end{remark}
\paragraph{Minkowski spaces} 
We first describe \emph{Minkowski spaces}, which is the embedding space where the hyperboloid model sits in as hypersurfaces.

The $(d+1)$-dimensional Minkowski space $\R^{1,d}$ is like the flat Euclidean space $\R^{1+d}$ except with a dot product that has a negative sign in the first coordinate. More precisely, $\R^{1,d}$ is the vector space $\R^{1+d}$ equipped with the indefinite, non-degenerate bilinear form
$$B \left( (t, x_1, \dots, x_d),(u, y_1, \dots, y_d) \right) = -tu + \sum_{i=1}^d x_i y_i,$$
which serves as the ``dot product.''

Like with Euclidean spaces, the quantity $B(\vec{v}, \vec{v})$ is called the \emph{(Minkowski) squared norm} of the vector $\vec{v}$. Two vectors $\vec{u}, \vec{v} \in \R^{1,d}$ are called \emph{orthogonal} if $B(\vec{u}, \vec{v}) = 0$. The \emph{orthogonal complement} of a linear subspace $V \subset \R^{1,d}$ is the set $V^\perp = \{\vec{w} \in \R^{1,d}: B(\vec{w}, \vec{v}) = 0 \text{ for every } \vec{v} \in V\}$, which still has dimension $\dim \R^{1,d} - \dim V$. However, unlike in Euclidean spaces, vectors in $\R^{1,d}$ can have negative or zero squared norms, and linear subspaces can intersect their orthogonal complements.

\paragraph{Types of vectors in $\R^{1,d}$} Vectors with negative, zero, and positive (Minkowski) squared norms are called \emph{time-like}, \emph{light-like}, and \emph{space-like}, respectively. Time-like and light-like vectors together form a solid double cone in $\R^{1,d}$.

If $\vec{v}$ is a time-like or light-like vector, we call it \emph{future-pointing} if its first coordinate is positive, otherwise we call it \emph{past-pointing}. It follows from Cauchy-Schwarz inequality that.
\begin{prop} \label{prop:future-pointing-dot-product-negative}
	If $\vec{u}, \vec{v}$ are future-pointing time-like or light-like vectors then $B(\vec{u}, \vec{v}) < 0$.
\end{prop}
A linear subspace $V$ of $\R^{1,d}$ is called \emph{space-like} if every non-zero vector in $V$ is space-like. In that case, the bilinear form $B(\cdot, \cdot)$ restricts to a positive-definite bilinear form on $V$, thus making it isometric to an Euclidean vector space. It follows from \cref{prop:future-pointing-dot-product-negative} that
\begin{prop} \label{prop:ortho-complement-time-like-is-space-like}
	The orthogonal complement of a time-like vector is a space-like linear subspace.
\end{prop}

\paragraph{The hyperboloid model of hyperbolic spaces} We are now ready to introduce the hyperboloid model $\H^d$. It sits inside $\R^{1,d}$ in a similar way to how the unit sphere $\S^d$ sits inside the Euclidean space $\R^{1+d}$.

\begin{definition} \label{defn:hyperboloid-model}
	The \emph{hyperboloid model} of $d$-dimensional hyperbolic spaces is the set $\H^d$ of future-pointing vectors in $\R^{1,d}$ with Minkowsi squared norm $-1$.
\end{definition}

\begin{remark}
    For the rest of \cref{appendix-sec:horo}, we will use $\H^d$ to denote this hyperboloid model as a subset of $\R^{1,d}$, and not the abstract hyperbolic space or its other models (e.g.\ Poincar\'e ball). This should not lead to any ambiguities because we will not work with any other model.
\end{remark}

The following properties of $\H^d$ further illustrate the analogy with spheres in Euclidean spaces.
\begin{prop} \label{prop:hyperboloid-model-orthogonal-to-radial}
	For every $\vec{x} \in \H^d$, the tangent space $T_{\vec{x}} \H^d$ of $\H^d$ at $\vec{x}$ is (parallel to) the orthogonal complement of $\vec{x}$ in $\R^{1,d}$.
\end{prop}
\begin{prop} \label{prop:hyperboloid-model-isometric-poincare}
	When restricted to each tangent space of $\H^d$, the bilinear form $B(\cdot, \cdot)$ is positive definite. This defines a Riemannian metric on $\H^d$ which has constant curvature $-1$. The \emph{stereographic projection}
	$$(t,x_1,\dots,x_d) \to \left(\frac{x_1}{1+t}, \dots, \frac{x_d}{1+t} \right)$$
	is an isometry between $\H^d$ and the Poincar\'e model. Its inverse map is given by
	$$(y_1, \dots, y_d) \to \frac{(1 + \sum_i y_i^2, 2y_1, \dots, 2y_d)}{1 - \sum_i y_i^2}.$$
\end{prop}
\begin{prop} \label{prop:hyperboloid-model-geodesic-submanifold}
	Every $k$-dimensional geodesic submanifold of $\H^d$ is the intersection of $\H^d$ with a $(k+1)$-dimensional linear subspace of $\R^{1,d}$.
\end{prop}

In the hyperboloid model, the ideal points of hyperbolic spaces are represented by light-like \emph{directions} (instead of individual vectors):
\begin{prop} \label{prop:hyperboloid-model-ideal-points}
	Each ideal point of $\H^d$ is represented by a $1$-dimensional linear subspace spanned by some $\vec{v} \in \R^{1,d}$ with $B(\vec{v},\vec{v}) = 0$. The map
	$$(t,x_1,\dots,x_d) \to \left(\frac{x_1}{t}, \dots, \frac{x_d}{t} \right)$$
	gives a correspondence between a light-like vector $\vec{v} \in \R^{1,d}$ and an ideal point $p \in \Sinf^{d-1}$ in the Poincar\'e model that is represented by $\operatorname{span}(\vec{v})$. This correspondence is compatible with the stereographic projection in \cref{prop:hyperboloid-model-isometric-poincare}. Its inverse map is given by
	$$(y_1, \dots, y_d) \to (1, y_1, \dots, y_d).$$
\end{prop}
We conclude this section by noting that geodesic hulls in the hyperboloid model are closely related to linear spans in Minkowski space:
\begin{prop} \label{prop:hyperboloid-model-geodesic-hull-is-linear-span}
    Let $S$ be a set of vectors that are either in $\H^d$ or represent ideal directions of $\H^d$. Then the geodesic hull of $S$ in $\H^d$ is the intersection of $\H^d$ with the linear span of $S$.
\end{prop}
In particular, since the spine $P$ in our setting (\cref{thm:horosphere-proj-well-defined}) is the geodesic hull of some ideal points, it is cut out by the linear span of the corresponding ideal directions.

\subsection{Computation of Horospherical Projections} \label{appendix-subsec:horosphere-proj-algo}
\begin{algorithm}[tb]
   \caption{Horospherical Projection}
   \label{alg:example}
\begin{algorithmic}
   \STATE {\bfseries Input:} point $x$, ideal points $\{p_1, \dots, p_K\}$, base point $b$.
   \STATE $z \gets \piMinkowski_{\operatorname{span}(p_1, \dots, p_K)} (x)$ \COMMENT{orthogonal projection in ambient space}
   \STATE $w \gets z / \| z \|$ \COMMENT{rescale to a vector on the hyperboloid}
   \STATE $u \gets b - w$ \COMMENT{subtraction in ambient space}
   \STATE $u \gets u - \piMinkowski_{\operatorname{span}(p_1, \dots, p_K)} (u)$ \COMMENT{orthogonal projection in ambient space}
   \STATE $u \gets u / \| u \|$ \COMMENT{make $u$ the unit tangent vector orthogonal to the spine $P$}
   \STATE $y \gets \exp_w(d_\H(x, w) \cdot u)$ \COMMENT{exponential map in $\H^d$}
   \STATE \RETURN $y$
\end{algorithmic}
\end{algorithm}

Now we describe how the three steps mentioned at the beginning of \cref{appendix-subsec:hyperboloid-model-background} can be implemented in the hyperboloid model. The results of this section are summarized in \cref{alg:example}. To transfer back and forth between the hyperboloid and Poincar\'e models, we use the formulas in \cref{prop:hyperboloid-model-isometric-poincare} and \cref{prop:hyperboloid-model-ideal-points}.

\paragraph{Step 1: Computing geodesic projections} We first describe how to compute the geodesic (or closest-point) projection from $\H^d$ to a geodesic submanifold $P$ in $\H^d$. This process is very similar to how geodesic projections work in Euclidean spheres: We first perform an orthogonal projection onto the linear subspace $V$ that cuts out $P$, then rescale the result to get a vector on $\H^d$. 

Generally, orthogonal projections in Minkowski spaces are very similar to those in Euclidean spaces. However, since vectors in $\R^{1,d}$ can have norm zero, the orthogonal projection $\piMinkowski_V$ is not well-defined for every subspace $V$. Thus, a little extra argument is needed.
\begin{prop}[Orthogonal projections onto time-containing linear subspaces] \label{prop:minkowski-orthogonal-proj}
	Let $V$ be a linear subspace of $\R^{1,d}$ that contains some time-like vectors. Then
	\begin{enumerate}
		\item $V \cap V^\perp = \{0\}$. Consequently, we can define a linear orthogonal projection $\piMinkowski_V: \R^{1,d} \to V$ as follows:
		
		Since $\R^{1,d} = V \oplus V^\perp$, every vector $\vec{x} \in \R^{1,d}$ can be uniquely written as $\vec{x} = \vec{z} + \vec{n}$ for some $\vec{z} \in V$ and $\vec{n} \in V^\perp$. Then, we let $\piMinkowski_V (\vec{x}) \coloneqq \vec{z}.$
		\item Let $A$ be a matrix whose column vectors form a linear basis of $V$. Let $B$ be the $(1+d)\times(1+d)$ symmetric matrix associated to the bilinear form $B$, i.e.\ the diagonal matrix with entries $(-1, 1, 1, 1, \dots, 1)$. Then $A^\top BA$ is non-singular, and the linear projection $\piMinkowski_V$ is given by $\vec{x} \to A (A^\top BA)^{-1} A^\top B\vec{x}$
		\item If $\vec{x}$ is a future-pointing time-like vector then so is $\piMinkowski_V(\vec{x})$.
	\end{enumerate}
\end{prop}
\begin{proof} $ $
	\begin{enumerate}
		\item Let $\vec{v}$ be a time-like vector in $V$. Then $\vec{v}^\perp$ is space-like by \cref{prop:ortho-complement-time-like-is-space-like}. Since $V \cap V^\perp \subset V^\perp \subset \vec{v}^\perp$, the intersection $V \cap V^\perp$ must be space-like. On the other hand, for every $\vec{w} \in V \cap V^\perp$, we have $B(\vec{w}, \vec{w}) = 0$, so $\vec{w}$ must be light-like. It follows that $V \cap V^\perp = \{0\}$. The $\R^{1,d} = V \oplus V^\perp$ part of the claim is a standard linear algebra fact.
		\item This follows from the same argument that deduces the formula for Euclidean orthogonal projections.
		\item To avoid cumbersome notations, let $\vec{z} = \piMinkowski_V(\vec{x})$. Then since $\vec{x} - \vec{z}$ is orthogonal to $V$ and in particular to $\vec{z}$, we have the Pythagorean formula
		$$B(\vec{z}, \vec{z}) + B(\vec{x} - \vec{z}, \vec{x} - \vec{z}) = B(\vec{x}, \vec{x}).$$
		On the other hand, since $\vec{x} - \vec{z}$ is orthogonal to $V$ and in particular orthogonal to some time-like vectors in $V$, by \cref{prop:ortho-complement-time-like-is-space-like}, it must be space-like. Thus, if $B(\vec{x}, \vec{x}) < 0$ then the above equation implies $B(\vec{z}, \vec{z}) < 0$, which means $\vec{z}$ is time-lile.
		
		Now either $\vec{z}$ or $-\vec{z}$ is future-pointing. In the latter case, since $\vec{x}$ is future-pointing, \cref{prop:future-pointing-dot-product-negative} implies $B(-\vec{z}, \vec{x}) < 0$. On the other hand, we have $B(\vec{z}, \vec{x}) = B(\vec{z}, \vec{z}) < 0$, contradicting the above inequality. Thus, $\vec{z}$ is future-pointing. \qedhere
	\end{enumerate}	
\end{proof}
\begin{prop}[Geodesic projections in $\H^d$] \label{prop:hyperboloid-geodesic-proj}
	Let $P$ be a geodesic submanifold of $\H^d$. Recall that by \cref{prop:hyperboloid-model-geodesic-submanifold}, $P = \H^d \cap V$ for some linear subspace $V$ of $\R^{1,d}$. Then for every $\vec{x} \in \H^d$, the geodesic projection $\piGeo_P (\vec{x})$ of $\vec{x}$ onto $P$ in $\H^d$ is given by
	$$\piGeo_P (\vec{x}) = \frac{\vec{z}}{\sqrt{-B(\vec{z}, \vec{z})}},$$
	where $\vec{z} = \piMinkowski_V(\vec{x})$ is the linear orthogonal projection of $\vec{x}$ onto $V$.
\end{prop}
\begin{proof}
	Again, to avoid cumbersome notations, let
	$$\vec{w} = \frac{\vec{z}}{\sqrt{-B(\vec{z}, \vec{z})}}.$$
	We will show that $\vec{w} = \piGeo_P(\vec{x})$. First, note that since $\vec{z}$ is a future-pointing time-like vector by \cref{prop:minkowski-orthogonal-proj}, $\vec{w} \in \H^d$. Since $\vec{w} \in V$, it follows that $\vec{w} \in P$.
	
	Let $W$ be the linear span of $\vec{w}$ and $\vec{x}$, so that $W \cap \H^d$ is the geodesic $\gamma$ in $\H^d$ that connects $\vec{w}$ and $\vec{x}$. 
	Note that $\vec{w}$ and $\vec{x} - \vec{z}$ form a basis of $W$. They are both orthogonal to the tangent space $T_{\vec{w}} P$ of $P$ at $\vec{w}$ because:
	\begin{itemize}
		\item By \cref{prop:hyperboloid-model-orthogonal-to-radial}, $\vec{w}$ is orthogonal to every tangent vector of $\H^d$ at $\vec{w}$.
		\item $T_{\vec{w}} P$ is contained in $V$, which is orthogonal to $\vec{x} - \vec{z}$.
	\end{itemize}
	Thus, $W$ is orthogonal to $T_{\vec{w}} P$. It follows that the geodesic $\gamma$ is orthogonal to $P$, which means $\vec{w} = \piGeo_P (\vec{x})$.	 
\end{proof}

\paragraph{Step 2: Finding orthogonal geodesic ray} Recall that if $P$ is a geodesic submanifold of $\H^d$ and $\vec{b} \in \H^d$ is a point not in $P$, then the geodesic hull $M$ of $P \cup \{\vec{b}\}$ in $\H^d$ is a geodesic submanifold of $\H^d$ with dimension $\dim P + 1$. Thus, $P$ cuts $M$ into two halves, and we denote the half that contains $\vec{b}$ by $P_b$.

Given a point $\vec{w} \in P$, there exists a unique geodesic ray $\alpha^+$ that starts at $\vec{w}$, stays on $P_b$, and is orthogonal to $P$. In this section, we describe how to compute $\alpha^+$.

Recall that by \cref{prop:hyperboloid-model-geodesic-submanifold}, $P = V \cap \H^d$ for some linear subspace $V$.

\begin{prop}
    The vector
    $$\vec{u} = (\vec{b} - \vec{w}) - \piMinkowski_{V} (\vec{b} - \vec{w})$$
    is tangent to $M$ and orthogonal to $P$ at $\vec{w}$. Furthermore, it points toward the side of $P_b$.
\end{prop}
\begin{proof}
    By construction, $\vec{u}$ is orthogonal to $V$. Note that $V$ contains both $\vec{w}$ and the tangent space $T_{\vec{w}} P$ of $P$ at $\vec{w}$. Thus, $\vec{u}$ is orthogonal to both of them. Together with \cref{prop:hyperboloid-model-orthogonal-to-radial}, this implies that $\vec{u}$ is a tangent vector of $\H^d$ at $\vec{w}$ that is orthogonal to $P$.
    
    Next, note that $M$ is the intersection of $\H^d$ with the linear span of $V \cup \{\vec{b}\}$, and since $\vec{b}$ and $\vec{w}$ both belong to this linear span, so does $\vec{u}$. Thus, since $\vec{u}$ is a tangent vector of $\H^d$ at $\vec{w}$, it must be tangent to $M$.
    
    By construction, $\vec{u}$ points from $\vec{w}$ toward the side of $\vec{b}$ instead of away from it. More rigorously, note that by essentially the same argument as above,
    $$\vec{a} \coloneqq (\vec{b} - \vec{w}) - \piMinkowski_{\vec{w}} (\vec{b} - \vec{w})$$
    is a tangent vector at $\vec{w}$ of the geodesic on $\H^d$ that goes from $\vec{w}$ to $\vec{b}$.
    Also, note that
    $$\vec{a} - \vec{u} = \piMinkowski_V (\vec{b} - \vec{w}) - \piMinkowski_{\vec{w}} (\vec{b} - \vec{w}).$$
    Since $V = T_{\vec{w}} P \oplus \operatorname{span}(\vec{w})$ is an orthogonal decomposition, this means
    $$\vec{a} - \vec{u} = \piMinkowski_{T_{\vec{w}} P} (\vec{b} - \vec{w}),$$
    which in particular implies $\vec{a} - \vec{u} \in T_{\vec{w}} P$. It follows that $\vec{u}$ is the projection of $\vec{a}$ onto the orthogonal complement of $T_{\vec{w}} P$ in $T_{\vec{w}} \H^d$.
    
    In an Euclidean vector space, the dot product of any vector with its projection onto any direction is positive unless the projection is zero. Since $T_{\vec{w}} \H^d$ is a space-like subspace, this statement applies to $\vec{a}$ and $\vec{u}$. Note that $\vec{u}$ cannot be zero because otherwise $\vec{b}$ would belong to $V$ and hence $P$. Thus, we conclude that $\vec{a} \cdot \vec{u} > 0$, which means that $\vec{u}$ points toward the side of $P_b$.
\end{proof}
Thus, $\vec{u}$ is the tangent vector at $\vec{w}$ of the desired ray $\alpha^+$. To compute points on this ray, we can use the exponential map at $\vec{w}$.

\paragraph{Step 3: The exponential map} Finally, given a distance $d$ and a tangent vector $\vec{u}$ at $\vec{w}$ of a geodesic ray $\alpha^+$ in $\H^d$, we need to compute the point on $\alpha^+$ that is of hyperbolic distance $d$ from $\vec{w}$. This is based on the following lemma:
\begin{lemma}
    Suppose that $\vec{u}$ is a unit tangent vector of $\H^d$ at $\vec{w}$. Then
    $$\alpha(t) = (\cosh t) \vec{w} + (\sinh t) \vec{u}$$
    is a unit-speed geodesic with $\alpha(0) = \vec{w}$ and $\dot{\alpha}(0) = \vec{u}$.
\end{lemma}
\begin{proof}
    Note that $\vec{w}$ and $\vec{u}$ are orthogonal by \cref{prop:hyperboloid-model-orthogonal-to-radial}. The facts that $\alpha(0)=\vec{w}$, $\dot{\alpha}(0) = \vec{u}$, and that $\alpha(t)$ is a unit-speed curve on $\H^d$ follow from simple, direct computations. To see that $\alpha(t)$ is a geodesic, note that it is the intersection of $\H^d$ with the linear subspace $\operatorname{span}(\vec{w}, \vec{u})$ and apply \cref{prop:hyperboloid-model-geodesic-submanifold}.
\end{proof}
It follows that
\begin{prop}
    If $\vec{u}$ is the tangent vector at the starting point $\vec{w}$ of a geodesic ray $\alpha^+$ in $\H^d$ then
    $$(\cosh d) \vec{w} + (\sinh d) \frac{\vec{u}}{\sqrt{B(\vec{u}, \vec{u})}}$$
    is the point on $\alpha^+$ that is of distance $d$ from $\vec{w}$. \hfill \qedsymbol
\end{prop}

From the results in this section, we conclude that \cref{alg:example} computes the horospherical projection $\piHoro_{b,P}(x)$.

\section{Geodesic Projection: Distortion Analysis}\label{appendix-sec:orthogonal-projection-distortion}
We show that in hyperbolic geometry, the geodesic projection $\piGeo_M$ (also called closest-point projection) to a geodesic submanifold $M$ always strictly decreases distances unless the inputs are already in $M$; furthermore, all paths of distance at least $r$ from $M$ get shorter by at least $\cosh r$ times under the projection. Note that most ideas and results in this section already exist in the literature. For a comprehensive treatment, see \cite{Thurston-notes}.

This section is organized as follows. First, we use basic hyperbolic trigonometry to prove a bound on projected distances onto geodesics in dimension two (\cref{prop:geodesic-proj-2D-distance-bound}). Then we generalize this bound to higher dimensions (\cref{thm:geodesic-proj-distance-bound}) and obtain an infinitesimal version of it (\cref{cor:geodesic-proj-infinitesimal-distortion}). The main theorem (\cref{thm:geodesic-proj-shrink-path}) then follows immediately.

\paragraph{Hyperbolic trigonometry in two dimensions} We first recall the laws of sine and cosine for triangle in hyperbolic planes:
\begin{prop} \label{prop:hyp-law-of-sin-cos}
	Consider any triangle in $\H^2$ with edge lengths $A, B, C$. Let $\alpha, \beta, \gamma$ be the angles at the vertices opposite to edges $A, B, C$. Then
	$$\frac{\sinh A}{\sin \alpha} = \frac{\sinh B}{\sin \beta} = \frac{\sinh C}{\sin \gamma},$$
	$$\cos \gamma = \frac{\cosh A \cosh B - \cosh C}{\sinh A \sinh B}.$$
\end{prop}

\begin{proof}
	See \cite{Thurston-notes}, section 2.6.	
\end{proof}
\begin{prop} \label{prop:hyp-right-quadrilateral}
	Consider any quadrilateral in $\H^2$ with vertices $x,y,z,w$ (in that order). Suppose that the angles at $z$ and $w$ are $90$ degrees. Then
	\begin{align*}
	\cosh(xy) = &\cosh(zw) \cosh(xw) \cosh(yz) \\
	&- \sinh(xw) \sinh(yz).
	\end{align*}
	
	Here $\cosh(xy)$ is a shorthand for $\cosh(d_\H (x,y))$.
\end{prop}
\begin{proof}	
	Applying the laws of sine to triangle $wxz$ gives	
	$$\frac{\sinh(xw)}{\sin \angle(zx,zw)} = \frac{\sinh(xz)}{\sin \angle(wx,wz)} = \sinh(xz),$$
	so
	$$\sin \angle(zx,zw) = \frac{\sinh(xw)}{\sinh(xz)}.$$
	On the other hand, applying the law of cosine to triangle $xyz$ gives
	$$\cos \angle(zx,zy) = \frac{\cosh(zx)\cosh(zy) - \cosh(xy)}{\sinh(zx)\sinh(zy)}.$$		
	Now note that $\sin \angle(zx,zw) = \cos \angle(zx,zy)$. Thus,
	$$\frac{\sinh(xw)}{\sinh(xz)} = \frac{\cosh(zx)\cosh(zy) - \cosh(xy)}{\sinh(zx)\sinh(zy)},$$
	which can be simplified to
	$$\cosh(zx)\cosh(zy) - \cosh(xy) = \sinh(xw)\sinh(zy).$$
	Finally, note that the law of cosine for triangle $wxz$ gives
	$$\frac{\cosh(wx)\cosh(wz) - \cosh(xz)}{\sinh(wx) \sinh(wz)} = \cos\angle(wx,wz) = 0,$$
	so $\cosh(xz) = \cosh(wx) \cosh(wz)$. Plugging this into the above equation gives
	\begin{align*}
	&\cosh(wx)\cosh(wz) \cosh(zy) - \cosh(xy) \\
	= &\sinh(xw) \sinh(zy),
	\end{align*}
	which can be arranged to get the desired identity.
\end{proof}

Now we are ready to prove a bound on the projected distances onto a geodesic in $\H^2$:
\begin{prop} \label{prop:geodesic-proj-2D-distance-bound}
	Let $L$ be a geodesic in $\H^2$. Consider any $x, y \in \H^2$ that are on the same side of $L$ and are of distances at least $r$ from $L$. Let $d = d_\H (x,y)$ and $d' = d_\H (\piGeo_L (x), \piGeo_L (y))$ denote the original and projected distances of $x$ and $y$. Then,
	$$\sinh(d'/2) \leq \frac{\sinh(d/2)}{\cosh r}.$$
\end{prop}
\begin{proof}
	Let $r_x = d_\H(x, \piGeo_L(x))$ and $r_y = d_\H (y, \piGeo_L(y))$. Then by \cref{prop:hyp-right-quadrilateral},
	$$\cosh d = \cosh d' \cosh r_x \cosh r_y - \sinh r_x \sinh r_y.$$
	Since $\cosh a \cosh b - \sinh a \sinh b = \cosh(a-b)$, this gives
	$$\cosh d = (\cosh d' - 1) \cosh r_x \cosh r_y + \cosh(r_x - r_y).$$
	Subtracting $1$ from both sides and using $\cosh (a) - 1 = 2 \sinh^2 (a/2)$ gives
	\begin{align*}
	&2\sinh^2 (d/2) \\
	= &2\sinh^2(d'/2)\cosh r_x \cosh r_y + 2 \sinh^2 ((r_x - r_y)/2).
	\end{align*}
	Thus,
	$$\sinh^2(d/2) \geq \sinh^2(d'/2) \cosh r_x \cosh r_y.$$
	Since $r_x, r_y \geq r$, this implies
	$$\sinh^2(d/2) \geq \sinh^2(d'/2)\cosh^2 r,$$
	which is equivalent to the desired inequality.
\end{proof}

\paragraph{General dimensions} Now we show that \cref{prop:geodesic-proj-2D-distance-bound} generalizes to higher dimensions.
\begin{theorem} \label{thm:geodesic-proj-distance-bound}
	Let $M \subset \H^d$ be a totally geodesic submanifold. Consider any $x, y \in \H^d$ that are of distances at least $r$ from $M$. Let $d = d_\H (x,y)$ and $d' = d_\H (\piGeo_M (x), \piGeo_M (y))$ denote the original and projected distances of $x$ and $y$. Then,
	$$\sinh(d'/2) \leq \frac{\sinh(d/2)}{\cosh r}.$$
\end{theorem}
\begin{proof}
	This is clearly true if $\piGeo_M(x) = \piGeo_M(y)$ or if both $x$ and $y$ belong to $M$. Thus, without loss of generality, we can assume that $\piGeo_M(x) \neq \piGeo_M(y)$ and $\piGeo_M(x) \neq x$.
	
	Let $L \subset M$ be the geodesic joining $\piGeo_M(x)$ and $\piGeo_M (y)$, so that $d_\H (x,M) = d_\H (x,L)$ and $d_\H (y,M) = d_\H (y,L)$. Let $L_x$ be the half-plane of $\GH(L \cup \{x\})$ that is bounded by $L$ and contains $x$. Note that if we rotate $y$ around $L$ until it hits $L_x$ at a point $\bar{y}$ then
	\begin{enumerate}
		\item[(i)] $\piGeo_L(\bar{y}) = \piGeo_L(y)$,
		\item[(ii)] $d_\H (\bar{y}, L) = d_\H (y, L)$,
		\item[(iii)] $\piHoro_{x,L}(y) = \bar{y}$ by the ``open book interpretation'' of horospherical projections (see the beginning of \cref{appendix-subsec:horosphere-proj-properties}).
	\end{enumerate}
	
	Thus, using (i), (ii) and applying \cref{prop:geodesic-proj-2D-distance-bound} to $L,x,\bar{y}$ gives,
	$$\sinh(d'/2) \leq \frac{\sinh(d_\H (x,\bar{y})/2 )}{\cosh r}.$$
	
	By (iii) and \cref{cor:horosphere-proj-non-expanding},
	$$d_\H (x, \bar{y}) \leq d_\H (x, y).$$
	
	Combining these two inequalities gives the desired inequality.
\end{proof}

Since $\sinh (t) \approx t$ as $t \to 0$, applying \cref{thm:geodesic-proj-distance-bound} as $y \to x$ gives

\begin{cor} \label{cor:geodesic-proj-infinitesimal-distortion}
	Under the geodesic projection $\piGeo_M (\cdot)$, every tangent vector $\vec{v}$ at $x$ gets at least $\cosh(d_\H (x, M))$ times shorter. $\qedsymbol$
\end{cor}

This implies the main theorem of this section:
\begin{theorem} \label{thm:geodesic-proj-shrink-path}
	Let $\gamma(t)$ be any smooth path in $\H^d$ that stays a distance at least $r$ away from $M$. Then
	$$\operatorname{length}(\piGeo_M(\gamma)) \leq \frac{\operatorname{length}(\gamma)}{\cosh r}.$$
	In particular, $\operatorname{length}(\piGeo_M(\gamma)) < \operatorname{length} (\gamma)$ unless $\gamma$ is entirely in $M$.
\end{theorem}
\begin{proof}
	Note that the velocity vector of $\piGeo_M(\gamma)$ is precisely $\piGeo_M(\dot{\gamma})$. Then, by \cref{cor:geodesic-proj-infinitesimal-distortion},
	\begin{align*}
	\operatorname{length}(\piGeo_M(\gamma)) &= \int \| \piGeo_M(\dot{\gamma}(t)) \|_\H dt \\
	&\leq \int \frac{\| \dot{\gamma}(t) \|_\H}{\cosh r} dt = \frac{\operatorname{length} (\gamma)}{\cosh r}.
	\end{align*}
	Since $\cosh r \geq 1$, this implies
	$$\operatorname{length}(\piGeo_M(\gamma)) \leq  \operatorname{length} (\gamma).$$
	This can only be an equality if the $\cosh r$ factor is constantly $1$, which means $\gamma$ stays in $M$.
\end{proof}
When $\gamma$ is a geodesic, this gives a proof of \cref{prop:geodesic-proj-shrink-path}:
\geoshrink*

\begin{remark}
The above proposition does \emph{not} imply that if $x$ and $y$ are of distances at least $r$ from $M$ then $d_\H(x,y)$ gets at least $\cosh r$ times shorter under the geodesic projection. This is because the geodesic segment between $x$ and $y$ can get closer to $M$ in the middle. This is inevitable when $x$ and $y$ are extremely far apart: For example, if $d_\H(x,M) = d_\H (y,M) = r$, then the triangle inequality implies
$$d_\H (\piGeo_M(x), \piGeo_M(y)) \geq d_\H (x,y) - 2r.$$
Hence, if $d_\H(x,y) \gg r$ then the shrink factor cannot be much smaller than $1$.
Generally, the rule of thumb is that geodesic projections shrink distances by a big factor if the input points are not very far from each other but quite far from the target submanifold.
\end{remark}

\begin{table*}
  \begin{center}
    \begin{small}
  \begin{tabular}{lcccccc}
    \toprule
      & \multicolumn{3}{c}{\textbf{Sarkar}} &
     \multicolumn{3}{c}{\textbf{PyTorch}} \\
     \cmidrule{2-7}
     {dimension} & $2$ & $10$ & $50$ & $2$ & $10$ & $50$ \\
     \midrule
     \textbf{Balanced Tree} & 0.088 & 0.136 & 0.152 & 0.097 & 0.022 & 0.022 \\
     \textbf{Phylo Tree} & 0.639 &0.639  & 0.638 &   0.292 & 0.286 & 0.284  \\
     \textbf{Diseases} & 0.247 & 0.212 & 0.210 & 0.126 &  0.054 & 0.055 \\
     \textbf{CS Ph.D.} &  0.388 & 0.384 & 0.380 & 0.189 & 0.140 & 0.139 \\
     \bottomrule
\end{tabular}
     \end{small}
    \end{center}
    \captionof{table}{Average distortion for embeddings learned with different methods. Sarkar refers to combinatorial embeddings learned with Sarkar's construction~\cite{sarkar2011low, sala2018representation} while PyTorch refers to embeddings learned using the PyTorch code from~\cite{gu2018learning}.}\label{tab:datasets}
\end{table*}
\section{Additional Experimental Details}\label{sec:app_exp}
We first provide more details about our experimental setup, including a description of how we implemented different methods and then provide additional experimental results.

\begin{table*}
\vskip 0.15in
\begin{center}
\begin{small}
\resizebox{\textwidth}{!}{\renewcommand{\arraystretch}{0.95}
  \begin{tabular}{lcccccccc}
    \toprule
    & \multicolumn{2}{c}{\textbf{Balanced Tree}} & \multicolumn{2}{c}{\textbf{Phylo Tree}}
    & \multicolumn{2}{c}{\textbf{Diseases}}  & \multicolumn{2}{c}{\textbf{CS Ph.D.}} \\
    \midrule
    & distortion ($\downarrow$) & variance ($\uparrow$) & distortion ($\downarrow$) & variance ($\uparrow$) & distortion ($\downarrow$) & variance ($\uparrow$) & distortion ($\downarrow$) & variance ($\uparrow$) \\
    \cmidrule(r){2-9}
    PCA & 0.91  & 1.09& 0.99  &0.28 & 0.95&0.75   & 0.96 & 1.25 \\
    tPCA & 0.83 & 4.01& 0.27 & 62.54& 0.64 & 35.45  & 0.55 & 86.10 \\
    PGA & 0.67 $\pm$ 0.013 & 11.15 $\pm$ 0.941  & 
    \underline{0.08} $\pm$ 0.001 & \underline{1110.30} $\pm$ 0.577 & 
    0.55 $\pm$ 0.006 & 56.82 $\pm$ 0.493 & 
    0.25 $\pm$ 0.002 & 270.28 $\pm$ 0.487\\
    BSA & 0.58 $\pm$ 0.018 & 16.74 $\pm$ 1.278  & \underline{0.08} $\pm$ 0.000 & 1110.05 $\pm$ 0.130 & 0.54 $\pm$ {0.008} & 57.73 $\pm$ 0.486 & 0.25 $\pm$ 0.001 & 270.98 $\pm$ 0.363\\
    hMDS & \underline{0.12}  & \underline{58.62}  & \underline{0.08} & 498.05  & \underline{0.32} &\textbf{568.22} & \underline{0.18} &\textbf{1241.34} \\
    \name{} & \textbf{0.06} $\pm$ \textbf{0.010} &\textbf{62.72} $\pm$ 0.924 &
    \textbf{0.01} $\pm$ 0.000 & \textbf{1190.27} $\pm$ 2.974 &
    \textbf{0.09} $\pm$ 0.010 & \underline{161.32} $\pm$ 1.185 &
    \textbf{0.05} $\pm$ 0.004 &    \underline{994.69} $\pm$ 321.359 \\
    \bottomrule
    \end{tabular}
    }
    \captionof{table}{Dimensionality reduction results on 10-dimensional combinatorial embeddings (Sarkar's construction) reduced to two dimensions. Results are averaged over 5 runs for non-deterministic methods. Best in \textbf{bold} and second best \underline{underlined}.}\label{tab:sarkar_reduction}
\end{small}
\end{center}
\vskip -0.1in
\end{table*}

\subsection{Datasets}
\paragraph{Embedding}
The datasets we use are structured as graphs with nodes and edges. 
We map these graphs to the hyperbolic space using different embedding methods. 
The datasets from~\cite{sala2018representation} can be found in the open-source implementation\footnote{\url{https://github.com/HazyResearch/hyperbolics}} and we compute hyperbolic embeddings for different dimensions using the PyTorch code from~\cite{gu2018learning}.
We also consider embeddings computed with Sarkar's combinatorial construction~\cite{sarkar2011low} and report the average distortion measure with respect to the original graph distances in~\cref{tab:datasets}.

For classification experiments, we follow the experimental protocol in~\cite{cho2019large} and use the embeddings provided in the open-source implementation.\footnote{\url{https://github.com/hhcho/hyplinear}}

\paragraph{Centering}
For simplicity, we center the input embeddings so that they have a Fr\'echet mean of zero. 
This makes computations of projections more efficient. 
To do so, we compute the Fr\'echet mean of input hyperbolic points using gradient-descent, and then apply an isometric hyperbolic reflection that maps the mean to the origin.\footnote{This reflection can be computed in closed-form using circle inversions, see for instance Section 2 in~\cite{sala2018representation}.}

\subsection{Implementation Details in the Poincar\'e Model}
\cref{appendix-subsec:horosphere-proj-algo} (in particular, \cref{prop:hyperboloid-geodesic-proj} and~\cref{alg:example}) describes how geodesic projections and horospherical projections can be computed efficiently in the hyperboloid model. 
However, because the Poincar\'e ball model is useful for visualizations and is popular in the ML literature, here we also give high-level descriptions of simple alternative methods to compute these projections directly in the Poincar\'e model.

This subsection is organized as follows. We first describe an implementation of geodesic projections in the Poincar\'e model (\cref{appendix-subsec:geodesic-proj-implementation-in-poincare}). We then describe our implementation of all baseline methods, which rely on geodesic projections (\cref{appendix-subsec:baseline-implementation}). Finally, we describe how \name{} could also be implemented in the Poincar\'e model (\cref{appendix-subsec:horosphere-proj-implementation-in-poincare}).

\subsubsection{Computing Geodesic Projections} \label{appendix-subsec:geodesic-proj-implementation-in-poincare}
Recall that geodesic projections map points to a target submanifold such that the projection of a point is the point on the submanifold that is closest to it.
In the Poincar\'e model, when the target submanifold goes through the origin, geodesic projections can be computed efficiently as follows.

Consider any geodesic submanifold $P$ that contains the origin in the Poincar\'e Ball: this must be a linear space since any geodesic that goes through the origin in this model is a straight-line.
The Euclidean reflection with respect to $P$ is a hyperbolic isometry (i.e.\ preserves distances).  
Given any point $x$, we can compute its Euclidean (and hyperbolic) reflection $r_P(x)$ with respect to $P$. Then, the (hyperbolic geometry) midpoint between $x$ and $r_P(x)$ belongs to $P$ and is the geodesic projection of $x$ onto $P$. There is a closed-form expression for this midpoint, which can be derived using Mob\"{i}us operations.

\subsubsection{Implementation of Baselines} \label{appendix-subsec:baseline-implementation}
We now detail our baseline implementation. 

\paragraph{PCA and tPCA} 
We used the Singular Value Decomposition (SVD) PyTorch implementation to implement both PCA and tPCA. 
Before using the SVD algorithm, tPCA maps points to the tangent space at the Fr\'chet mean using the logarithmic map. 
Having centered the data, this mapping can be done using the logarithmic map at the origin which is simply: 
\[
\log_\mathbf{o}(x)=\arctanh(||x||)\frac{x}{||x||}.
\]

\paragraph{PGA and BSA}
Both PGA and BSA rely on geodesic projections. 
When the data is centered, the target submanifolds in BSA and PGA are simply linear spaces going through the origin. We can therefore compute the geodesic projections in these methods using the computations described above. 
To test their sensitivity to the base point, we also implemented two baselines that perturb the base point, by adding Gaussian noise. 

\paragraph{hMDS}
To implement hMDS, we simply implemented the formulas in the original paper (see Algorithm 2 in~\cite{sala2018representation}).

\paragraph{hAE} 
For hyperbolic autoencoders, we used two fully-connected layers following the approach of~\cite{ganea2018hyperbolic}. 
One layer was used for encoding into the reduced number of dimensions, and one layer was used for decoding into the original number of dimensions. 
We trained these networks by minimizing the reconstruction error, and used the intermediate hidden representations as low-dimensional representations. 

\subsubsection{\name{} Implementation} \label{appendix-subsec:horosphere-proj-implementation-in-poincare}
The definition of horospherical projections onto $K > 1$ directions involves taking the intersection of $K$ horospheres $S(p_1, x) \cap \dots \cap S(p_K, x)$ (\cref{subsec:horo_proj_high}).
Although \cref{appendix-subsec:horosphere-proj-algo} shows that horospherical projections can be computed efficiently in the hyperboloid model without actually computing these intersections, it is also possible to implement \name{} by directly computing these intersections.
This can be done in a simple iterative fashion.

For example, in the Poincar\'e ball, horospheres can be represented as Euclidean spheres.
Note that the intersection of two $d$-dimensional Euclidean hyperspheres $S(p_0, r_0)$ and $S(p_1, r_1)$ is a $d-1$-dimensional hypersphere which can be represented by a center, a radius, and $d-1$-dimensional subspace.
The radius is uniquely determined by the radii of the original spheres $r_0, r_1$ and the distance between their centers $\|p_0-p_1\|$ and analyzing the 2-dimensional case, for which there is a simple closed formula.
The center can similarly be found as a weighted combination of the centers $p_0, p_1$ based on the formula for the 2D case.
Lastly, this $d-1$-dimensional subspace can be represented by noting it is the orthogonal space to the vector $p_1-p_0$.
Thus the intersection of these $d$-dimensional hyperspheres can be easily computed, and this process can be iterated to find the intersection of $K$ spheres $S(p_i, r_i)$.
\begin{figure}
    \centering
    \includegraphics[width=3in]{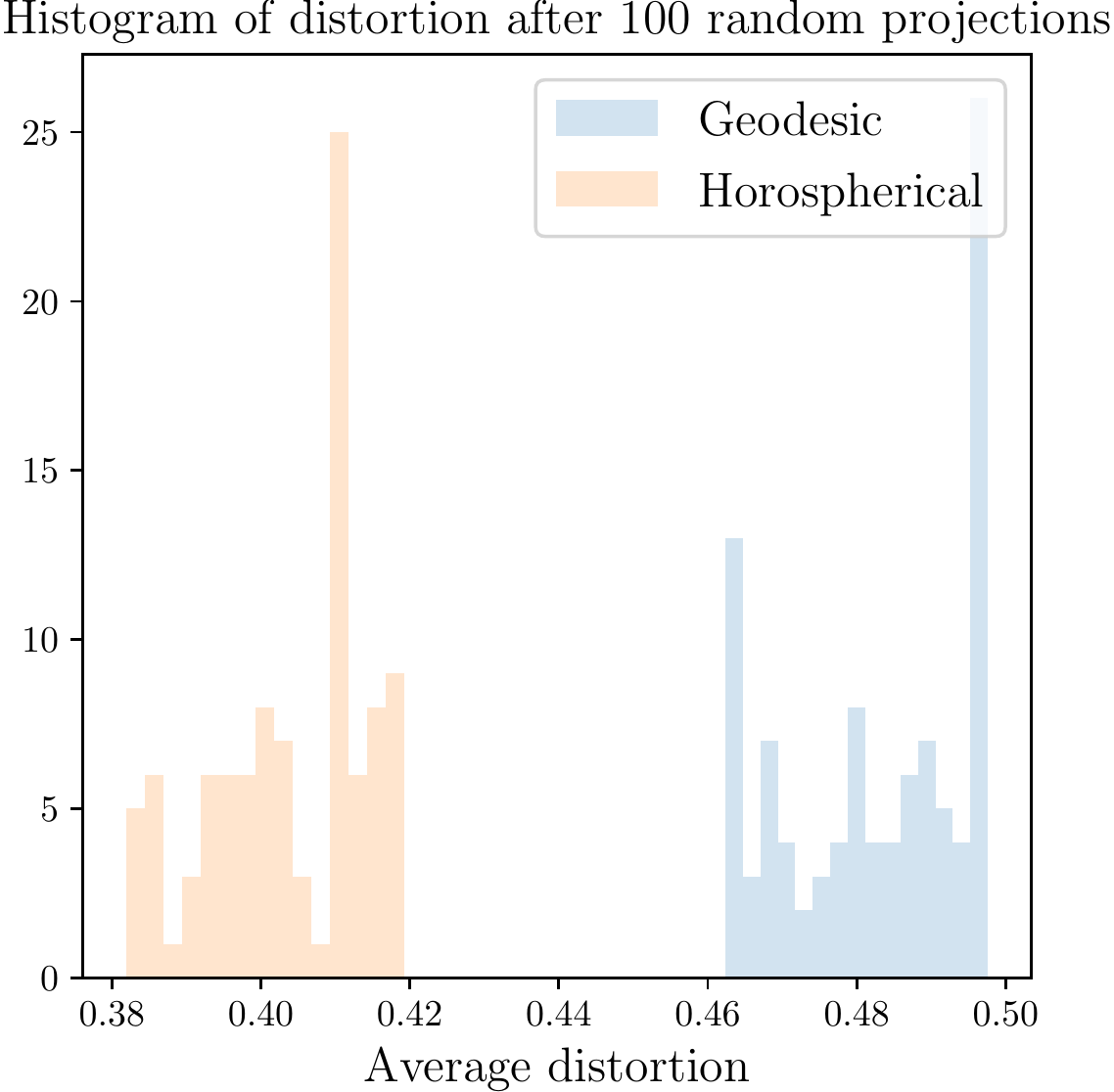}
    \caption{Histogram of average distortion computed for the projections of 1000 random hyperbolic points onto 100 random directions.
    Best in \textbf{bold} and second best \underline{underlined}.}
    \label{fig:distortion_hist}
\end{figure}

\subsection{Additional Experimental Results}
\subsubsection{Distortion Analysis}
We first analyze the average distortion incurred by horospherical and geodesic projections on a toy example with synthetically-generated data.
We generate 1000 points the Poincar\'e disk by sampling tangent vectors from a multivariate Gaussian, and mapping them to the disk using the exponential map at the origin. 
We then sample 100 random ideal points (directions) in the disk, and consider the corresponding straight-line geodesics pointing towards these directions. 
We project the datapoints onto these geodesics and measure the average distortion from before and after projection and visualize the results in a histogram~(\cref{fig:distortion_hist}).
As we can see, Horospherical projections achieve much lower distortion than geodesic projections on average, suggesting that these projections may better preserve information such as distances in higher-dimensional datasets.

\subsubsection{Dimensionality Reduction Results}

\paragraph{Sarkar embeddings}
\cref{tab:pytorch_reduction} showed dimensionality reduction results for embeddings learned with optimization. 
Here, we consider the same reduction experiment on combinatorial embeddings and report the results in~\cref{tab:sarkar_reduction}. 
The results confirm the trends observed in~\cref{tab:pytorch_reduction}: \name{} outperforms baseline methods, with significant improvements on distance preservation.

\paragraph{More dimension/component configurations}
We consider the reduction of 50-dimensional PyTorch embeddings of the Diseases and CS Ph.D. datasets. 
We plot average distortion and explained Fr\'echet variance for different number of components in~\cref{fig:dim_plots}.
\name{} significantly outperforms all previous generalizations of PCA.  
\name{} also outperforms hMDS, which is a competitive baseline, but not a PCA method as discussed before.  
    \begin{figure*}
        \centering
        \begin{subfigure}[b]{0.475\textwidth}
            \centering
            \includegraphics[width=\textwidth]{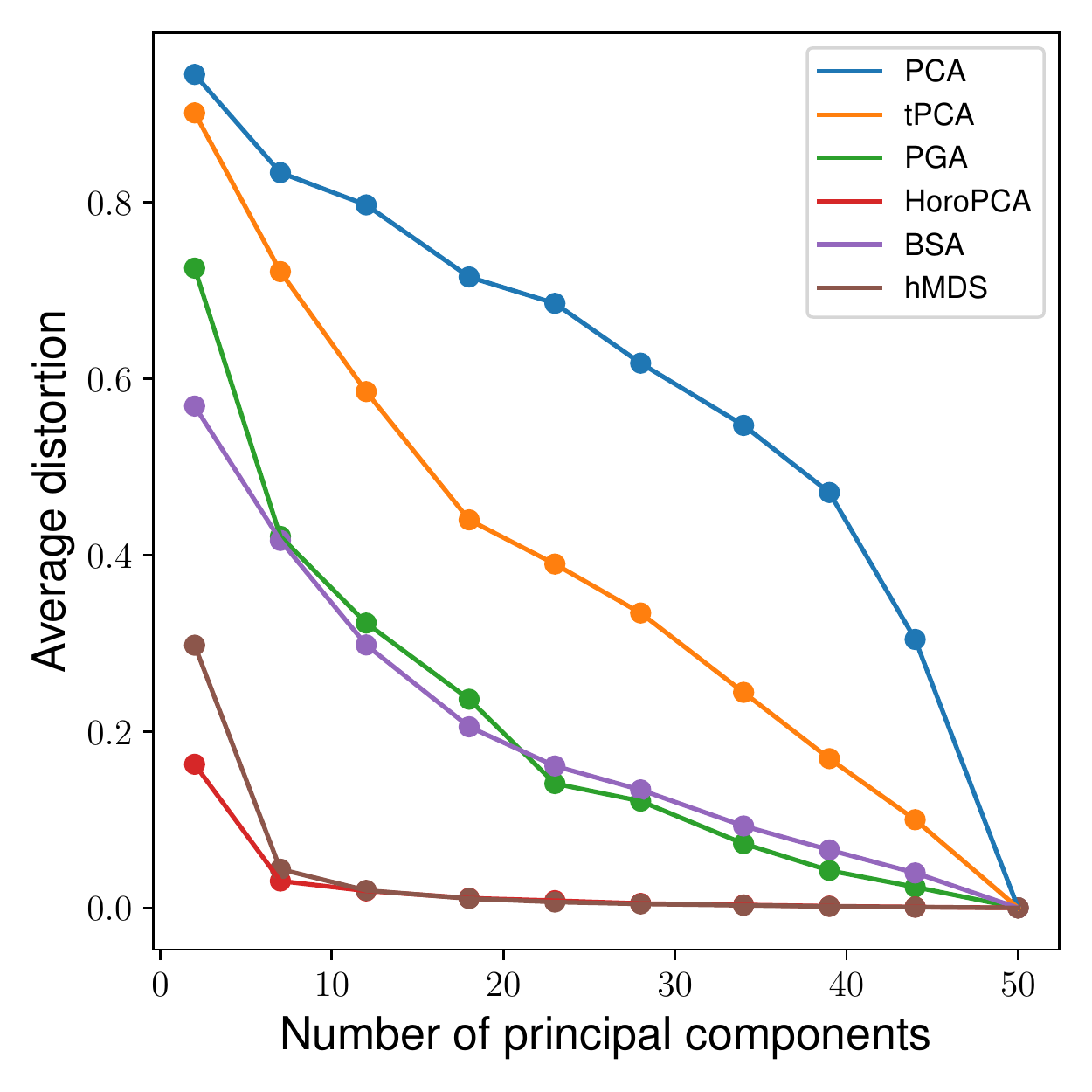}
            \caption[]%
            {{\small Diseases average distortion.}}
        \end{subfigure}
        \hfill
        \begin{subfigure}[b]{0.475\textwidth}  
            \centering 
            \includegraphics[width=\textwidth]{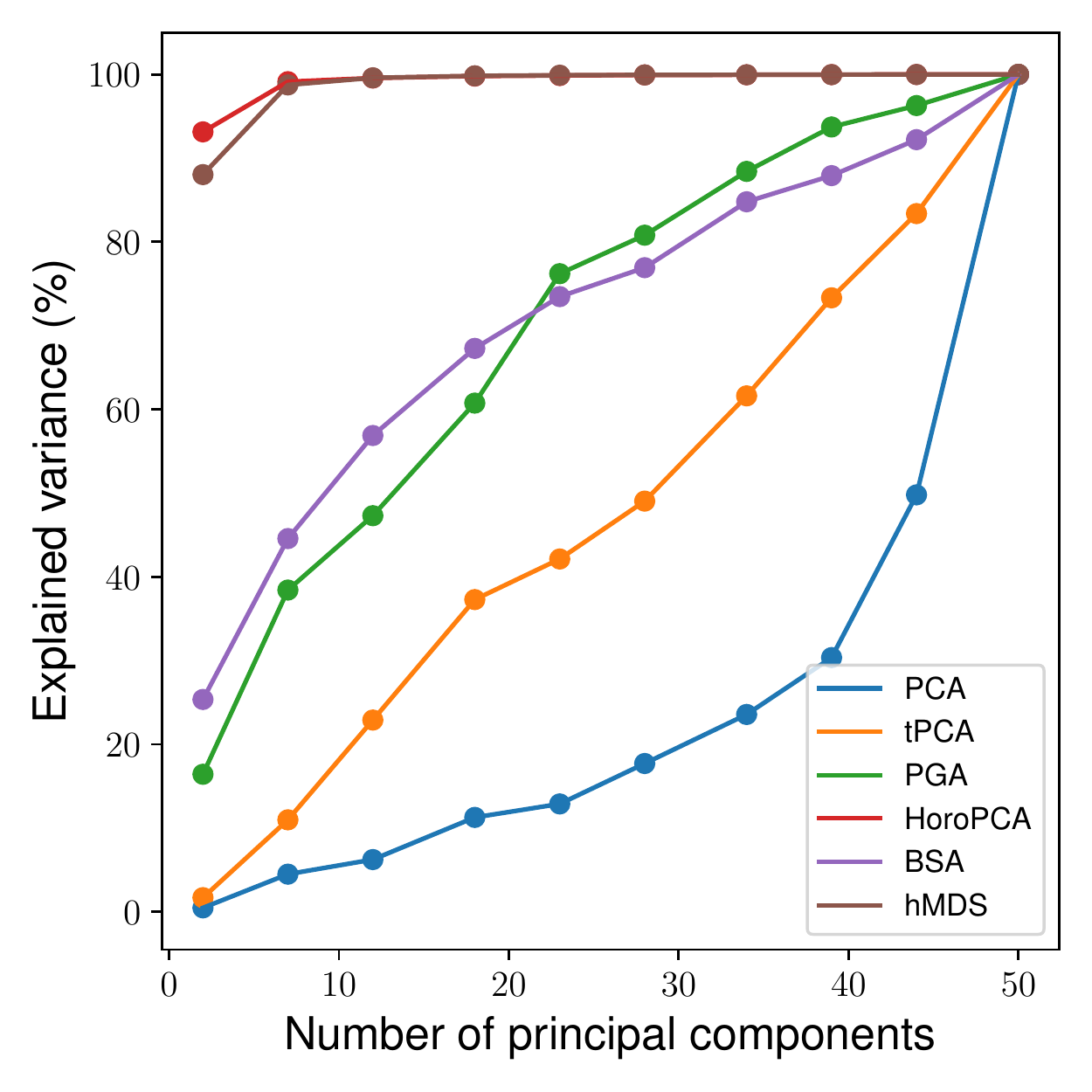}
            \caption[]%
            {{\small Diseases explained Fr\'echet variance.}}  
        \end{subfigure}
        \vskip\baselineskip
        \begin{subfigure}[b]{0.475\textwidth}   
            \centering 
            \includegraphics[width=\textwidth]{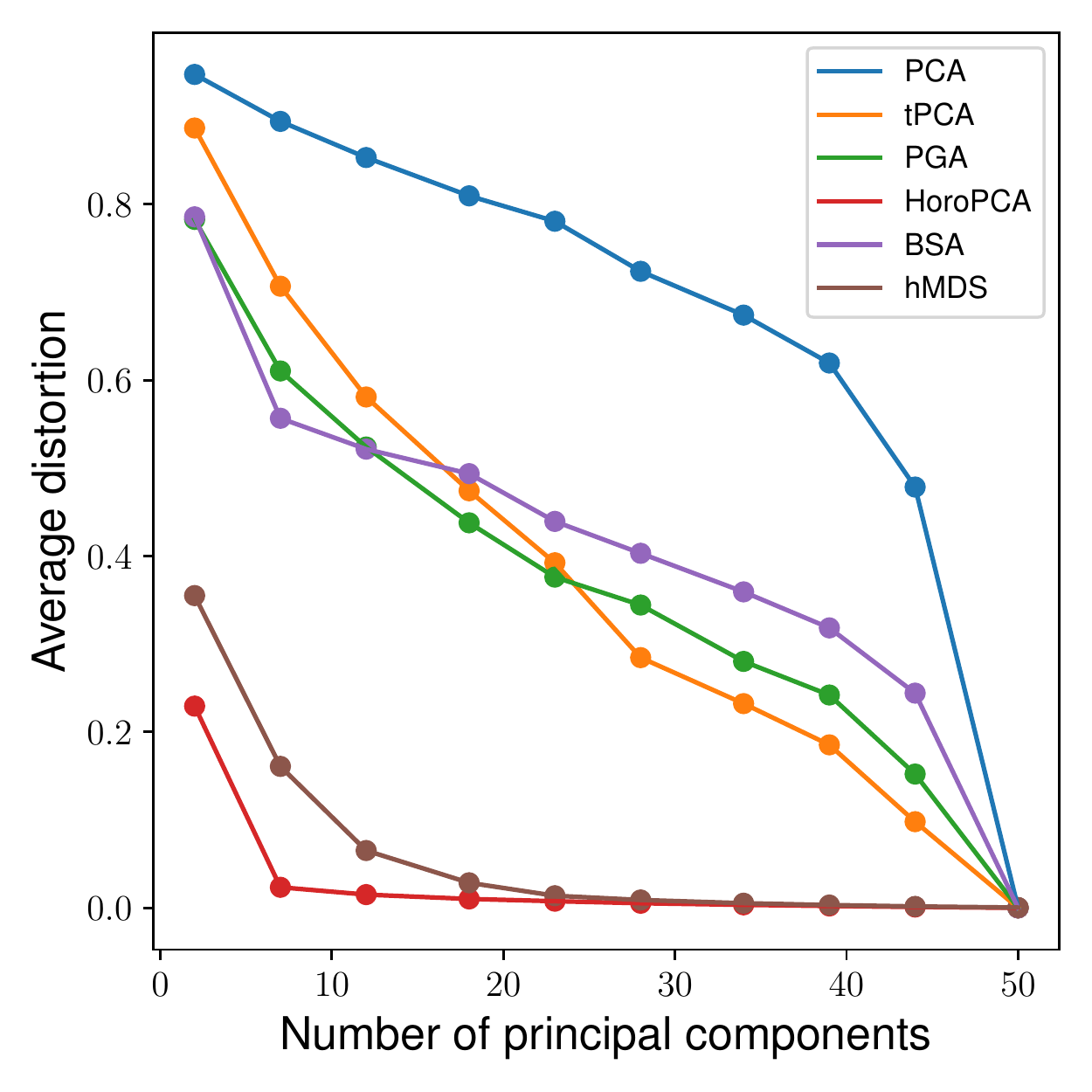}
            \caption[]%
            {{\small CS Ph.D. average distortion.}}  
        \end{subfigure}
        \hfill
        \begin{subfigure}[b]{0.475\textwidth}   
            \centering 
            \includegraphics[width=\textwidth]{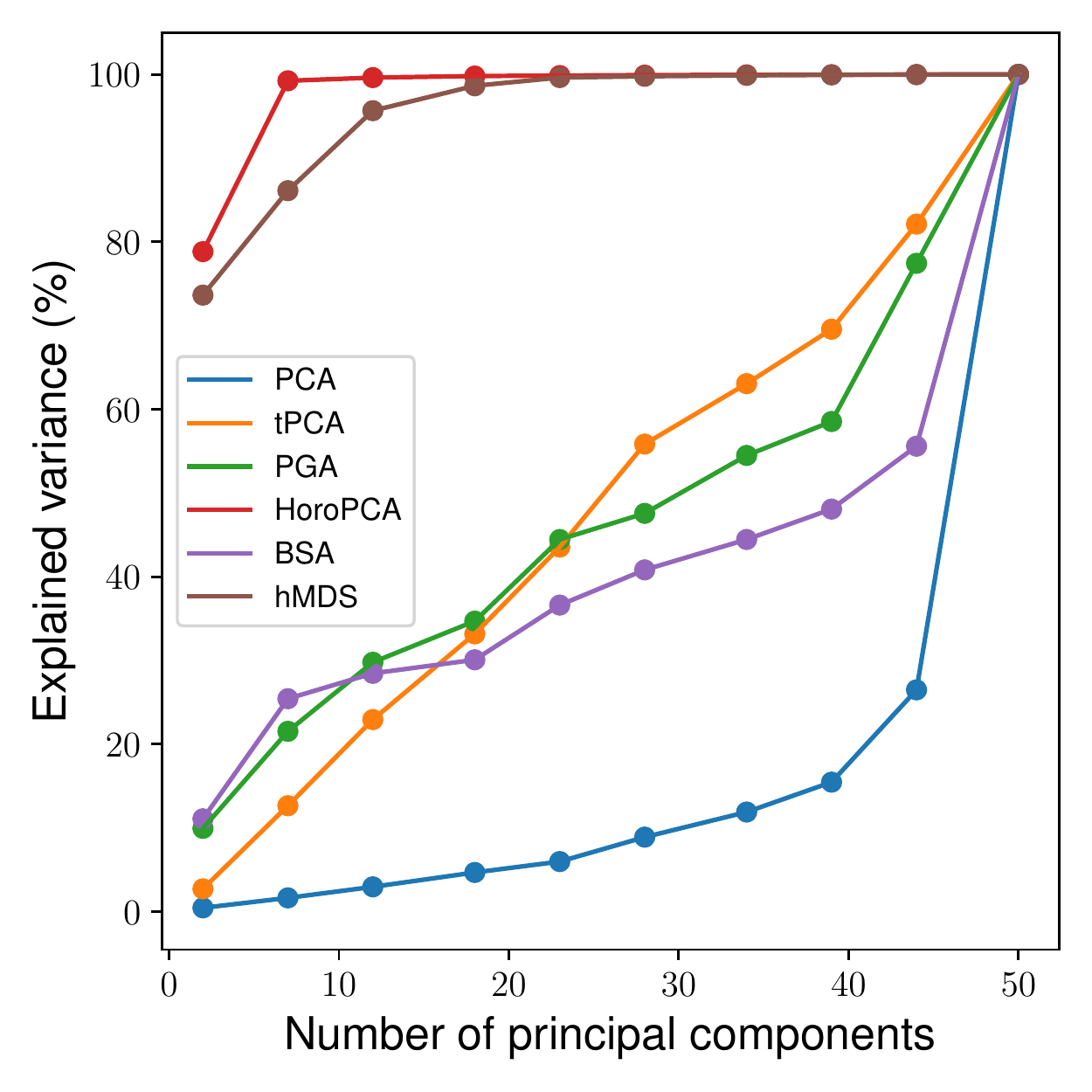}
            \caption[]%
            {{\small CS Ph.D. explained Fr\'echet variance.}}     
        \end{subfigure}
        \caption[]
        {\small Average distortion and explained Fr\'echet variance for different numbers of principal components.} 
        \label{fig:dim_plots}
    \end{figure*}
\begin{figure*}[t]
        \centering
        \begin{subfigure}[b]{0.49\textwidth}
            \centering
            \includegraphics[width=0.92\textwidth]{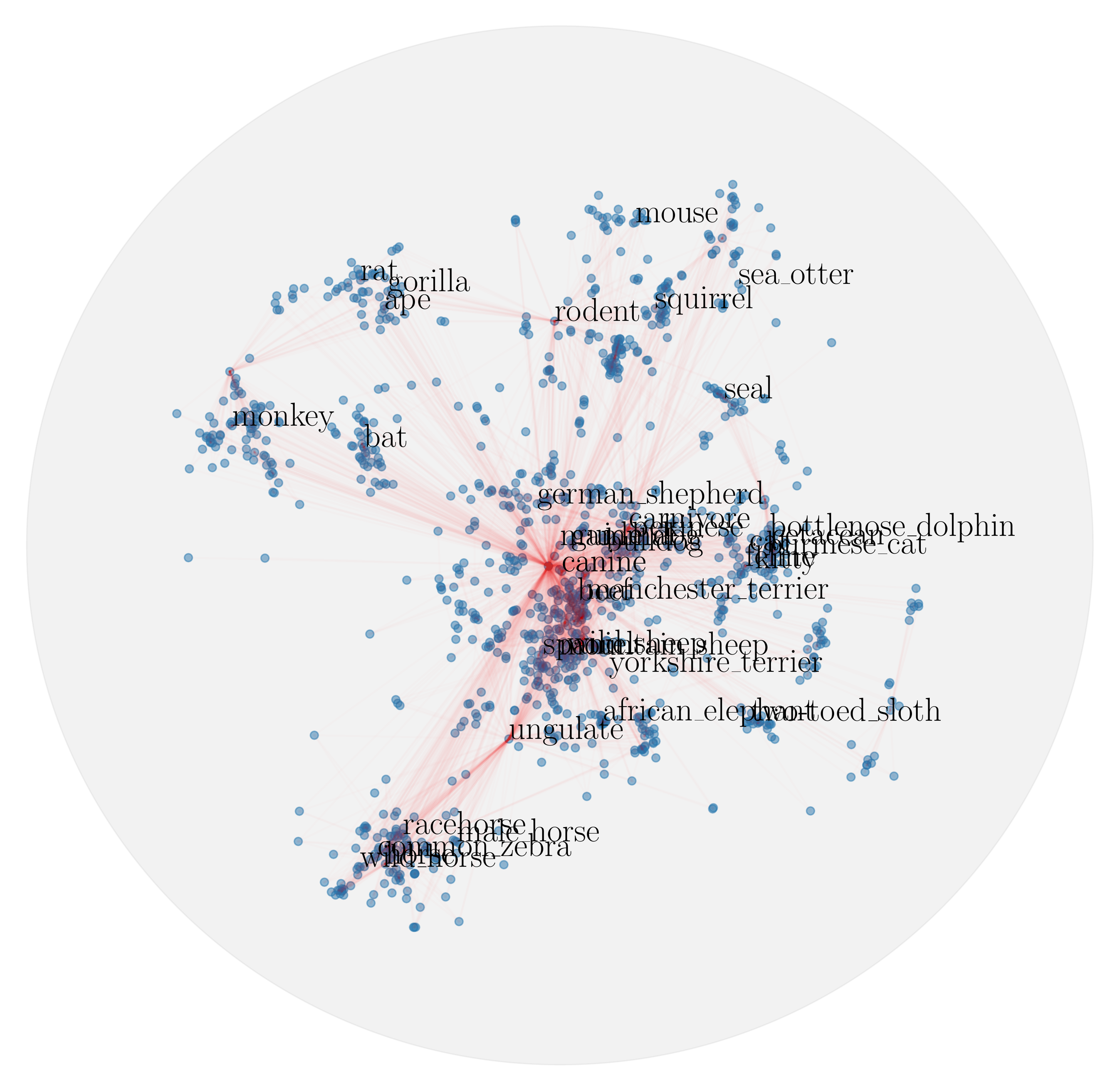}
            \caption[]%
            {{\small PCA (average distortion: 0.942)}}
        \end{subfigure}
        \begin{subfigure}[b]{0.49\textwidth}  
            \centering 
            \includegraphics[width=\textwidth]{images/pga_wordnet.pdf}
            \caption[]%
            {{\small PGA (average distortion: 0.534)}}  
        \end{subfigure}
        \vskip\baselineskip
        \begin{subfigure}[b]{0.49\textwidth}   
            \centering 
            \includegraphics[width=\textwidth]{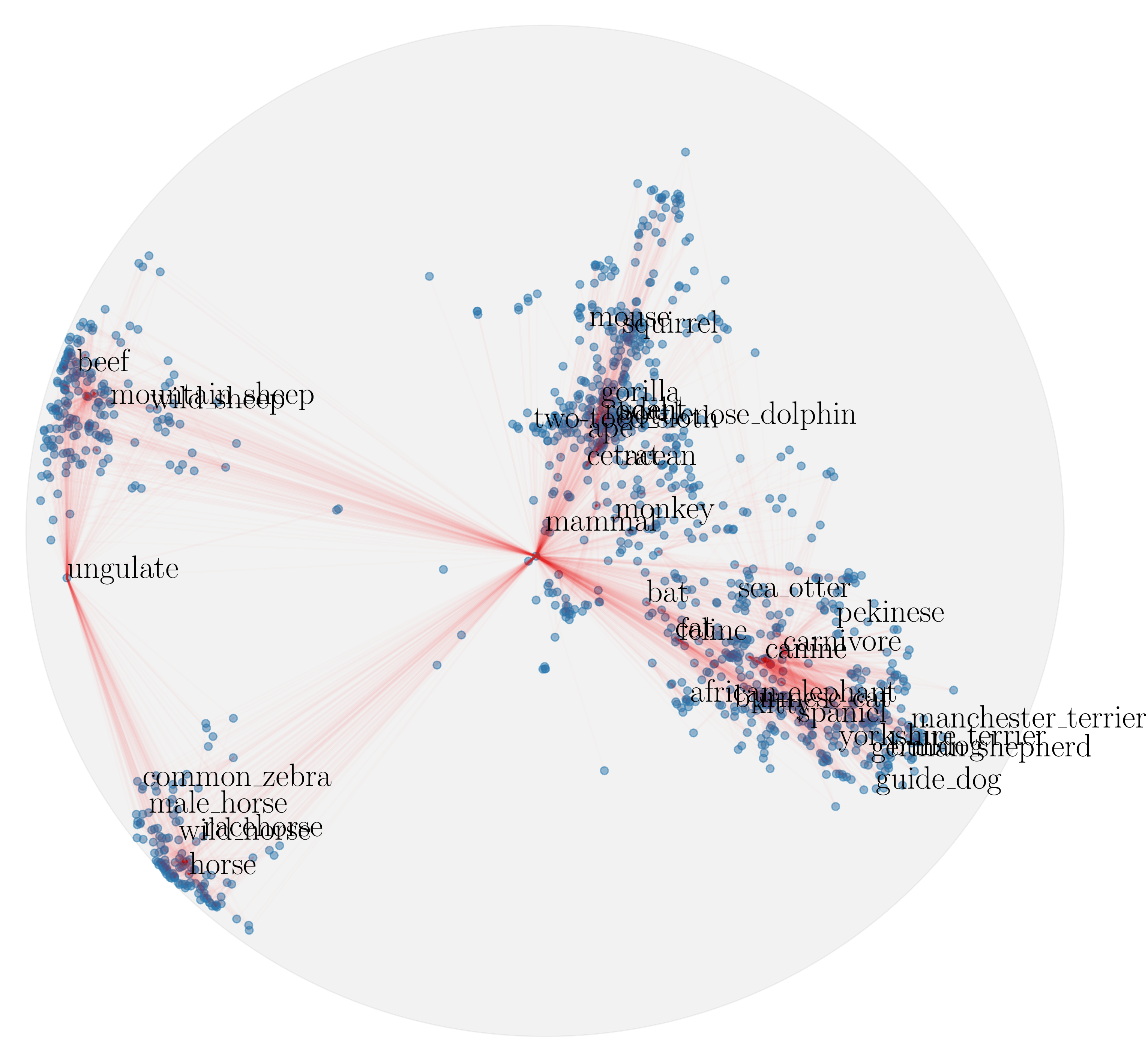}
            \caption[]%
            {{\small BSA (average distortion: 0.532)}}  
        \end{subfigure}
        \begin{subfigure}[b]{0.49\textwidth}   
            \centering 
            \includegraphics[width=\textwidth]{images/horopca_wordnet.pdf}
            \caption[]%
            {{\small \name{} (average distortion: 0.078)}}     
        \end{subfigure}
        \caption[]
        {\small Visualization of embeddings of the WordNet mammal subtree computed by reducing 10-dimensional Poincar\'e embeddings~\cite{nickel2017poincare}.}
        \label{fig:poincare_wordnet_full}
    \end{figure*}

\end{document}